\documentclass[letterpaper]{article} % DO NOT CHANGE THIS
\usepackage{aaai2026}  % DO NOT CHANGE THIS
\usepackage{times}  % DO NOT CHANGE THIS
\usepackage{helvet}  % DO NOT CHANGE THIS
\usepackage{courier}  % DO NOT CHANGE THIS
\usepackage[hyphens]{url}  % DO NOT CHANGE THIS
\usepackage{graphicx} % DO NOT CHANGE THIS
\usepackage{natbib}  % DO NOT CHANGE THIS AND DO NOT ADD ANY OPTIONS TO IT
\usepackage{caption} % DO NOT CHANGE THIS AND DO NOT ADD ANY OPTIONS TO IT
\usepackage{algorithm}
\usepackage[noend]{algorithmic}
\usepackage{amsmath}
\usepackage{amssymb}
\usepackage{amsthm}
\usepackage{color}
\usepackage{xcolor}
\theoremstyle{plain}
\newtheorem{dfn}{Definition}
\newtheorem{lem}{Lemma}
\newtheorem{thm}{Theorem}

\newtheorem{prb}{Problem}

\newtheorem{exm}{Example}

\newcommand{\IFLINE}[2]{\STATE\algorithmicif\ #1\ \algorithmicthen\ #2}

\usepackage{newfloat}
\usepackage{listings}
\DeclareCaptionStyle{ruled}{labelfont=normalfont,labelsep=colon,strut=off} % DO NOT CHANGE THIS
\floatstyle{ruled}
\newfloat{listing}{tb}{lst}{}
\floatname{listing}{Listing}
\title{I-CAM-UV: Integrating Causal Graphs over Non-Identical Variable Sets\\Using Causal Additive Models with Unobserved Variables}
\author {
    Hirofumi Suzuki\textsuperscript{\rm 1},
    Kentaro Kanamori\textsuperscript{\rm 1},
    Takuya Takagi\textsuperscript{\rm 1},
    Thong Pham\textsuperscript{\rm 2, \rm 5}\\
    Takashi Nicholas Maeda\textsuperscript{\rm 3, \rm 5},
    Shohei Shimizu\textsuperscript{\rm 2, \rm 4, \rm 5}
}
\affiliations {
    \textsuperscript{\rm 1}Fujitsu Limited\\
    \textsuperscript{\rm 2}Faculty of Data Science, Shiga University\\
    \textsuperscript{\rm 3}Computer Centre, Gakushuin University\\
    \textsuperscript{\rm 4}SANKEN, The University of Osaka\\
    \textsuperscript{\rm 5}RIKEN AIP\\
    suzuki-hirofumi@fujitsu.com, k.kanamori@fujitsu.com, takagi.takuya@fujitsu.com, thong-pham@biwako.shiga-u.ac.jp,\\
    takashi.maeda@gakushuin.ac.jp, shohei-shimizu@ds.sanken.osaka-u.ac.jp
}
\begin{document}

\maketitle

\begin{abstract}
Causal discovery from observational data is a fundamental tool in various fields of science.
While existing approaches are typically designed for a single dataset, we often need to handle multiple datasets with non-identical variable sets in practice.
One straightforward approach is to estimate a causal graph from each dataset and construct a single causal graph by overlapping.
However, this approach identifies limited causal relationships because unobserved variables in each dataset can be confounders, and some variable pairs may be unobserved in any dataset.
To address this issue, we leverage Causal Additive Models with Unobserved Variables (CAM-UV) that provide causal graphs having information related to unobserved variables.
We show that the ground truth causal graph has structural consistency with the information of CAM-UV on each dataset.
As a result, we propose an approach named I-CAM-UV to integrate CAM-UV results by enumerating all consistent causal graphs.
We also provide an efficient combinatorial search algorithm and demonstrate the usefulness of I-CAM-UV against existing methods.
\end{abstract}


%%%%%%%%%%%%%%%%%%%%%%%%%%%%%%%%%%%%%%%%%%%%%%%%%%
% Introduction
%%%%%%%%%%%%%%%%%%%%%%%%%%%%%%%%%%%%%%%%%%%%%%%%%%
\section{Introduction}\label{sec:intro}

%%% 因果探索の重要性と未観測変数
Identifying causal relationships among variables is a fundamental task in various fields of science.
While the most effective approach is a randomized controlled trial, such an approach is often difficult due to cost, ethical, and technical reasons.
Thus, the mainstream is \emph{causal discovery} from purely observational data~\citep{Spirtes2000causation}: estimating causal relationships as a causal graph, often assumed to be a directed acyclic graph (DAG).
Recent works in causal discovery span diverse fields, including environmental science~\citep{Runge2023causal,Fu2025Environmental}, biology~\citep{Fu2025nature,Smith2025Multi}, and materials and drug sciences~\citep{Campomanes2014origin,Zhou2025tcbbio}.

In many of the causal discovery examples, the number of variables is relatively small---often around 10---because of limits in what can be measured or easily understood.
On the other hand, most of these studies assume that there are no unobserved confounders, even though such hidden common causes can exist even in small systems. 
This highlights the need for causal discovery methods that can handle unobserved confounding, especially in small to medium systems.

Moreover, in many situations, multiple datasets share common objectives but differ in measurement settings such as observed variables.
If we can leverage such datasets, it is expected to identify more causal relationships over just using a single dataset.
Thus, in this study, we consider causal discovery from multiple datasets with non-identical variable sets~\cite{Tillman14BHMK,Mooij20JCI}.

%%% 非共通変数に対する手法の動向（ION,Huagn20）
Despite the usefulness and demand, existing approaches are typically designed for a single dataset, assuming no unobserved confounders, and only a few approaches tackle the above scenario.
ION~\citep{Tillman2009ION}, IOD~\citep{Tillman2011IOD}, and COmbINE~\citep{Triantafillou2010COmbINE} construct partial ancestral graphs (PAGs) over the integrated variable set by utilizing conditional independence information.
In contrast, CD-MiNi~\citep{Huang2020CD-MiNi} constructs a full DAG by restricting the causal model to linear non-Gaussian cases.
However, these approaches have the following limitations: PAGs contain uncertain information about causal relationships, and CD-MiNi is not suitable for non-linear cases.

%%% 提案
Aiming to estimate full DAGs and handle a more expressive model, we focus on the Causal Additive Model (CAM)~\citep{Buhlmann14CAM} that captures non-linear causal relationships.
One straightforward approach is to estimate a causal graph from each dataset by CAM and integrate them into a single graph.
However, this approach identifies limited causal relationships because some variables in each dataset can be confounded with each other.
Moreover, if some variable pairs are unobserved in any dataset, their causal relationships can not be identified by direct estimation from the datasets.

%%% 貢献まとめ
To address this issue, we leverage the Causal Additive Model with Unobserved Variables (CAM-UV)~\citep{Maeda21UAI} that allows for the presence of unobserved variables.
CAM-UV provides structural information on causal graphs related to unobserved variables (unobserved causal path and unobserved backdoor path).
We show that the ground truth causal graph has structural consistency with the provided information of CAM-UV on each dataset.
Because multiple different DAGs satisfy such characterization, we propose an approach to enumerate possible DAGs and name the approach Integrating CAM-UV (I-CAM-UV).
Moreover, we show an efficient algorithm for I-CAM-UV and confirm its usefulness via experiments.
Our contributions are summarized as follows:
\begin{itemize}
\item We propose an enumeration approach named I-CAM-UV to obtain integrated causal graphs of CAM-UV over multiple datasets with non-identical variable sets.
\item We show that I-CAM-UV recovers the ground truth DAG when CAM-UV results have no enormous errors and the integrated variable set has no further confounders.
\item We propose an efficient I-CAM-UV algorithm with a best-first strategy based on a monotonic inconsistency of DAGs and demonstrate its usefulness via experiments.
\end{itemize}

%%%%%%%%%%%%%%%%%%%%%%%%%%%%%%%%%%%%%%%%%%%%%%%%%%
% Related Work
%%%%%%%%%%%%%%%%%%%%%%%%%%%%%%%%%%%%%%%%%%%%%%%%%%
\section{Related Work}
%%% 段落1．複数データセットで変数が一部重複のケースを考えてる手法は少ないよ
While there are a variety of similar but slightly different datasets in the real world, almost all studies on causal discovery are specialized for a single dataset, and only a few studies tackle multiple datasets with non-identical variable sets~\citep{Tillman2009ION,Triantafillou2010COmbINE,Tillman2011IOD,Huang2020CD-MiNi}.
Our study challenges that field and aims to allow for richer and more practical causal analysis.

%%% 段落2. 制約ベースとか関数ベースとかいろいろあって、我々は関数ベースに着目しているよ
Causal discovery methods can be categorized into three approaches: \emph{functional model-based}~\citep{shimizu2006,Hoyer09NIPS}, \emph{constraint-based}~\citep{Spirtes91PC,Spirtes95FCI}, and \emph{score-based}~\citep{chickering2002} methods.
They are not mutually exclusive and can be combined.
In this study, we focus on a non-linear functional model-based approach called the Causal Additive Model with Unobserved Variables (CAM-UV)~\citep{Maeda21UAI}.
Causal Additive Model (CAM)~\citep{Buhlmann14CAM} assumes non-linear causal functions where effects from each variable can be separated linearly, which is a special case of Additive Noise Model (ANM)~\citep{Hoyer09NIPS}.
CAM-UV identifies causal relationships as much as possible and detects latent causal/backdoor paths containing unobserved variables.
Leveraging the such ability of CAM-UV, we take a constraint-based approach to integrate CAM-UV results by a combinatorial search method.

%%% 段落3. 我々の手法に最も関連するライバルはHuang20で，〜〜が違うよ
CD-MiNi~\citep{Huang2020CD-MiNi} takes a functional model-based approach over multiple datasets with non-identical variable sets, similar to our method.
It assumes a linear non-Gaussian acyclic model (LiNGAM)~\citep{shimizu2006} and utilizes a continuous optimization algorithm based on the overcomplete independent component analysis to learn a causal graph~\citep{Ding2019OICA}.
While CD-MiNi handles linear causal relationships derived from LiNGAM, we handle non-linear causal relationships via CAM.

%%% 段落4. テクニックが微妙に似てる既存研究としてIONやCOmbINEがあるから一応紹介するけど，解いてる問題が違うよ（なので比較しないよ）
ION~\citep{Tillman2009ION}, IOD~\citep{Tillman2011IOD}, and COmbINE~\citep{Triantafillou2010COmbINE} focus on constraint-based approaches such as PC and FCI over multiple datasets with non-identical variable sets.
They consider partial ancestral graphs (PAGs) representing Markov equivalence classes (MECs) of causal graphs and enumerate all possible PAGs having structural consistency with PC or FCI results.
On the other hand, we consider enumerating not PAGs but full DAGs.
That is, while the existing methods output a PAG set implicitly representing DAGs of MECs, our method outputs a single explicit DAG set.

%%%%%%%%%%%%%%%%%%%%%%%%%%%%%%%%%%%%%%%%%%%%%%%%%%
% Preliminaries
%%%%%%%%%%%%%%%%%%%%%%%%%%%%%%%%%%%%%%%%%%%%%%%%%%
\section{Preliminaries}

% 変数集合と真の因果グラフがあって
Let $V = \{v_i\}_{i=1}^d$ be a variable set, and we consider causal relationships over $V$ as a graph structure.
Let $G^* = (V, A^*)$ be the ground truth causal graph, which is a DAG with directed edges $A^* \subset V^2$.
Each $(v_i, v_j) \in A^*$ means $v_i$ is a direct cause (parent) of $v_j$ and let $P^*_j := \{v_i \mid (v_i, v_j) \in A^*\}$ be the set of parents of $v_j$.

% DAGに従ってCAMが定義されて
We assume that the causal relationships are formulated by CAM: $v_i = \sum_{v_j \in P^*_i} f_j^{(i)}(v_j) + n_i$~\citep{Buhlmann14CAM} where each $f_j^{(i)}$ is a non-linear function, $n_i$ is the noise term on $v_i$, and $n_1, \ldots, n_d$ are independent of each other.
We have an algorithm to estimate $G^*$ on CAM from a given dataset of $V$.
However, in practice, unobserved variables $U \subset V$ may lead to incomplete estimation.

% 未観測変数があるときCAM-UVが定義される
Let $P_j := P^*_j \setminus U$ and $Q_j := P^*_j \cap U$.
CAM-UV is formulated as $v_i = \sum_{v_j \in P_i} f_j^{(i)}(v_j) + \sum_{v_k \in Q_i} f_k^{(i)}(v_k) + n_i$~\citep{Maeda21UAI}.
Given a dataset of $V \setminus U$, the estimation algorithm of CAM-UV aims to obtain a mixed graph
$G = (V \setminus U, A, N)$ where $N$ is an undirected edge set such that, if and only if $\{v_i, v_j\} \in N$ holds, the causal relationship between $v_i$ and $v_j$ is not identified due to the existence of an \emph{unobserved causal path} (UCP) or \emph{unobserved backdoor path} (UBP) between $v_i$ and $v_j$ on $G^*$.%a potential path between $v_i$ and $v_j$ on $G^*$.%

% UCPの定義
\begin{dfn}[Unobserved Causal Path (UCP)]
A path from $v_i$ to $v_j$ is called a unobserved causal path iff its form is $v_i \rightarrow \cdots \rightarrow v_k \rightarrow v_j$ where $v_k \in U$.
\end{dfn}

% UBPの定義
\begin{dfn}[Unobserved Backdoor Path (UBP)]
A path between $v_i$ and $v_j$ is called a unobserved backdoor path iff its form is $v_i \leftarrow v_x \leftarrow \cdots \leftarrow v_a \rightarrow \cdots \rightarrow v_y \rightarrow v_j$ where $v_x, v_y \in U$.
It is allowed to be $v_a = v_x$ and $v_a = v_y$.
\end{dfn}

% 図示
Figure \ref{fig:ucp-ubp} shows an illustration of UCP and UBP.
Figure \ref{fig:cam-uv} shows an example of CAM-UV result.

% Figure: UCP/UBP
\begin{figure}[t!]
\centering
\includegraphics[width=0.99\linewidth]{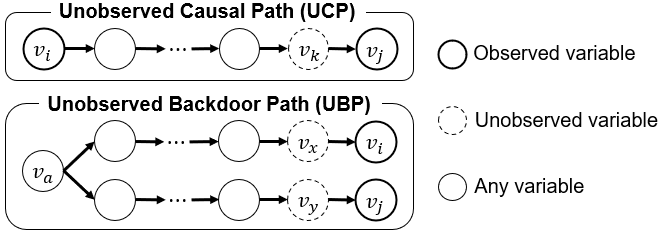}
\caption{Illustration of unobserved causal path (UCP) and unobserved backdoor path (UBP).}
\label{fig:ucp-ubp}
\end{figure}

% Figure: CAM-UV
\begin{figure}[t!]
\centering
\includegraphics[width=0.99\linewidth]{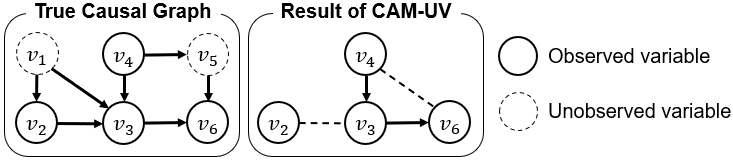}
\caption{
Example of CAM-UV.
A UCP $v_4 \rightarrow v_5 \rightarrow v_6$ and a UBP $v_2 \leftarrow v_1 \rightarrow v_3$ are found.
The dashed lines indicate undirected edges of the resulting mixed graph.
}
\label{fig:cam-uv}
\end{figure}

% Figure: inputs
\begin{figure}[t!]
\centering
\includegraphics[width=0.99\linewidth]{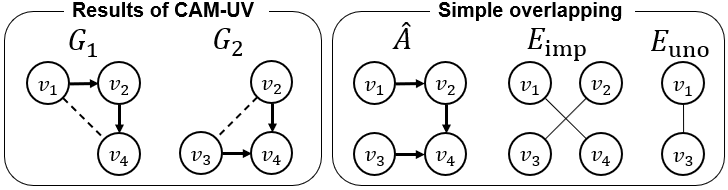}
\caption{
Example of inputs on a scenario we consider (left) and their simple overlapping (right).
}
\label{fig:inputs}
\end{figure}

%%%%%%%%%%%%%%%%%%%%%%%%%%%%%%%%%%%%%%%%%%%%%%%%%%
% Integrating CAM-UV
%%%%%%%%%%%%%%%%%%%%%%%%%%%%%%%%%%%%%%%%%%%%%%%%%%
\section{Integrating CAM-UV}
% 関連性のある複数のデータセットを分析するシナリオを考えて
We handle a scenario to estimate a causal graph from multiple datasets with non-identical variable sets via CAM-UV.
Given $m$ datasets of different variable sets $V_1, \ldots, V_m \subset V$ where $k \neq l \Leftrightarrow V_k \neq V_l$, we assume each variable set has a common variable with at least one other variable set.
Since, if some datasets share common sample sets, it is better to combine them into a single dataset, we assume the datasets also differ from each other in their sample sets.
Moreover, we assume all the datasets follow the same causal graph, but functional forms and data distributions may be different.
Regarding variables not in $V_k$ as unobserved for $k$-th dataset, CAM-UV provides a mixed graph $G_k = (V_k, A_k, N_k)$.
Our goal is to estimate a causal graph over $\hat{V} = \bigcup_{k=1}^m V_k$ from the given CAM-UV results $G_1, \ldots, G_m$.

% 単純な重ね合わせでは不足することがある
The identified causal relationships on $G_1, \cdots, G_m$ can be naively overlapped into a single graph:
Let $\hat{A} := \bigcup_{k=1}^m A_k$ and $\hat{G} := (\hat{V}, \hat{A})$.
However, the causal relationships for some variable pairs are not identified yet (see Figure \ref{fig:inputs} for an example):
$E_\mathrm{imp} := \left( \bigcup_{k=1}^m N_k \right) \setminus \{\{v_i, v_j\} \mid (v_i, v_j) \in \hat{A}\}$ and $E_\mathrm{uno} := \{\{v_i, v_j\} \subseteq \hat{V} \mid \forall k, \{v_i, v_j\} \not\subseteq V_k\}$.
The causal relationships on $E_\mathrm{imp}$ are impossible to identify in any dataset due to the existence of a UCP/UBP.
Moreover, since any variable pair in $E_\mathrm{uno}$ is not observed simultaneously in any dataset, it is even more challenging to identify their causal relationships.

%%%%%%%%%%%%%%%%%%%%%%%%%
% Ideal Situations
%%%%%%%%%%%%%%%%%%%%%%%%%
\subsection{Problem Definition for Ideal Situations}
% 各々の結果に矛盾しないDAGを見つけられたら嬉しい
To address the above limitations, we propose a novel approach named \emph{Integrating CAM-UV} (I-CAM-UV) that finds causal graphs satisfying consistency with respect to the given CAM-UV results.
First, we define the consistency of causal graphs based on UCP/UBP.
We then show that the ground truth causal graph satisfies the consistency.

Let $E := E_\mathrm{imp} \cup E_\mathrm{uno}$.
For any $k \in \{1, \ldots, m\}$, let $I_k := \{\{v_i, v_j\} \subseteq V_k \mid \{v_i, v_j\} \notin N_k\}$ be the variable pairs with the identified causal relationships on $V_k$.
Let $\mathrm{UP}_{G,V_k}(v_i,v_j)$ be the set of all the UCPs and UBPs between $v_i$ and $v_j$ on a graph $G$ over $\hat{V}$ where $U = \hat{V} \setminus V_k$.

\begin{dfn}
\label{dfn:consistent}
Given CAM-UV results $G_1, \ldots, G_m$, let $\tilde{A}$ be a directed edge set obtained by assigning directions to or excluding edges in $E$.
We call $\tilde{G} = (\hat{V}, \hat{A} \cup \tilde{A})$ consistent if and only if, for all $k \in \{1,\ldots,m\}$, $\mathrm{UP}_{\tilde{G},V_k}(v_i, v_j) = \emptyset$ holds for any $\{v_i, v_j\} \in I_k$ and $\mathrm{UP}_{\tilde{G},V_k}(v_i, v_j) \neq \emptyset$ holds for any $\{v_i, v_j\} \in N_k$.
\end{dfn}

We here state what \emph{ideal situations} are.
Given CAM-UV results $G_1, \ldots, G_m$, we are in an ideal situation if we have $\hat{V} = \bigcup_{k=1}^m V_k = V$, i.e., $\hat{V}$ has no unobserved variables, and $G_1, \ldots, G_m$ have no estimation error.

\begin{thm}
\label{thm:main-thm}
$G^*$ is consistent under any ideal situation.
\end{thm}

Theorem \ref{thm:main-thm} implies that the ground truth causal graph corresponds to a consistent DAG.
However, because there may exist multiple consistent DAGs, I-CAM-UV aims to solve the following problem.
\begin{prb}
\label{prb:ideal}
Given CAM-UV results $G_1, \ldots, G_m$, the problem is to enumerate all the consistent DAGs.
\end{prb}

In other words, given a set of mixed graphs $G_1, \ldots, G_m$, it implicitly represents a set of consistent DAGs $\mathcal{G}$, and we aim to explicitly enumerate them $\tilde{G}_1, \ldots, \tilde{G}_{|\mathcal{G}|} \in \mathcal{G}$.
From this perspective, I-CAM-UV might appear only to extract the implicitly represented DAGs.
However, a crucial aspect of I-CAM-UV lies in its ability to orient unobserved variable pairs, $E_\mathrm{uno}$, which allows for the identification of the causal relationships between variables where the direct estimation via observed datasets is difficult.

% Figure: Ideal Examples
\begin{figure}[t!]
\centering
\includegraphics[width=0.99\linewidth]{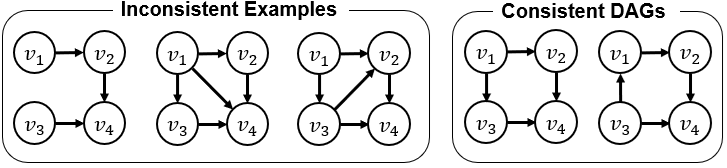}
\caption{
Example of I-CAM-UV on the input of Figure \ref{fig:inputs}.
}
\label{fig:outputs}
\end{figure}

\begin{figure}[t!]
\centering
\includegraphics[width=0.99\linewidth]{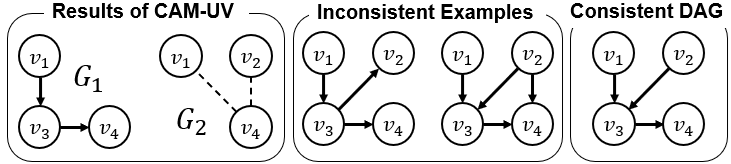}
\caption{
Example of I-CAM-UV resulting a unique DAG.
}
\label{fig:unique}
\end{figure}

\begin{exm}
When the input is the form shown in Figure \ref{fig:inputs}, I-CAM-UV works as shown in Figure \ref{fig:outputs}.
Focusing on $G_1$, we need $v_1 \rightarrow v_3$ or $v_1 \leftarrow v_3$ to obtain the UCP $v_1 \rightarrow v_3 \rightarrow v_4$ or the UBP $v_1 \leftarrow v_3 \rightarrow v_4$.
Focusing on $G_2$, we exclude $\{v_1, v_4\}$ to avoid the UBP $v_2 \leftarrow v_1 \rightarrow v_4$.
A similar discussion holds when $G_1$ and $G_2$ are swapped.
As a result, we obtain two consistent DAGs that have a directed edge between the unobserved variable pair $\{v_1, v_3\}$.
\end{exm}

\begin{exm}
In Figure \ref{fig:unique}, focusing on $G_2$, while $G_1$ induces the UCP $v_1 \rightarrow v_3 \rightarrow v_4$, we still need the UCP $v_2 \rightarrow v_3 \rightarrow v_4$ or the UBP $v_2 \leftarrow v_3 \rightarrow v_4$.
We fix $v_2 \rightarrow v_3$ to avoid the UCP $v_1 \rightarrow v_3 \rightarrow v_2$.
We then exclude $(v_2, v_4)$ to avoid the UBP $v_3 \leftarrow v_2 \rightarrow v_4$ with respect to $G_1$.
As a result, we obtain a unique consistent DAG which has a directed edge between the unobserved variable pair $\{v_2, v_3\}$.
\end{exm}

%%%%%%%%%%%%%%%%%%%%%%%%%
%  Realistic Situations
%%%%%%%%%%%%%%%%%%%%%%%%%
\subsection{Problem Relaxation for Realistic Situations}
% 推定ミスを許与する仕組みも取り入れる
While I-CAM-UV works well under any ideal situations, in practical use cases, unobserved variables may still exist ($\hat{V} \subset V$) and CAM-UV may cause estimation errors.
These may not yield any consistent graph.
In fact, estimating CAM-UV is currently not complete in theory to obtain proper mixed graphs, although the first algorithm~\cite{Maeda21UAI} is modified in the literature~\cite{pham2025camuvx} to improve accuracy.
Therefore, we define a new problem with a relaxed graph consistency requirement by introducing an \emph{inconsistency cost} which counts the number of variable pairs with inconsistent (non-)existence of UCP/UBP on each $V_k$.

\begin{dfn}
\label{dfn:inconsistency-cost}
For any graph $\tilde{G}$, to represent inconsistent variable pairs on $\tilde{G}$ for each $k$, let $\bar{I}_k(\tilde{G}) := \{\{v_i, v_j\} \in I_k \mid \mathrm{UP}_{\tilde{G},V_k}(v_i,v_j) \neq \emptyset\}$ and $\bar{N}_k(\tilde{G}) := \{\{v_i, v_j\} \in N_k \mid \mathrm{UP}_{\tilde{G},V_k}(v_i,v_j) = \emptyset\}$.
We define the inconsistency cost of $\tilde{G}$ as $C(\tilde{G}) := \sum_{k=1}^m (|\bar{I}_k(\tilde{G})| + |\bar{N}_k(\tilde{G})|)$.
\end{dfn}

\begin{prb}
\label{prb:relaxed}
Given CAM-UV results $G_1, \ldots, G_m$ and the user parameter $b \in \mathbb{Z}_{\geq 0}$, the problem is to enumerate all the DAGs with the inconsistency cost less than or equal to $C^* + b$ where $C^*$ is the minimum inconsistency cost of DAGs.
\end{prb}

\begin{exm}
In Figure \ref{fig:realistic}-(a), we have the estimation error of $\{v_1,v_2\} \in N_2$ on CAM-UV results because the ground truth has no edge between $v_1$ and $v_2$.
This case gives three DAGs with the minimum inconsistency cost of 1.
All the DAGs are inconsistent for $\{v_1,v_2\} \in N_2$ because no UCP/UBP between $v_1$ and $v_2$ will never occur.
However, I-CAM-UV can recover the ground truth DAG as the rightmost DAG.
Note that I-CAM-UV can not identify which of the enumerated graphs is the ground truth.
\end{exm}

\begin{exm}
\label{ex:can-not-recover}
In Figure \ref{fig:realistic}-(b), we show an example such that I-CAM-UV can not recover the ground truth DAG even though we solve the relaxed problem.
We have the estimation error of $(v_3, v_4) \notin A_1$ where $(v_3, v_4) \in A_1$ is true.
Nevertheless, this case has three false DAGs with the inconsistency cost of 0.
Moreover, even if we seek DAGs with non-zero inconsistency costs, the ground truth DAG does not appear because $(v_3, v_4) \notin \hat{A}$ and $\{v_3, v_4\} \notin E$.
\end{exm}

% Figure: Realistic Example
\begin{figure}[t!]
\centering
\includegraphics[width=0.99\linewidth]{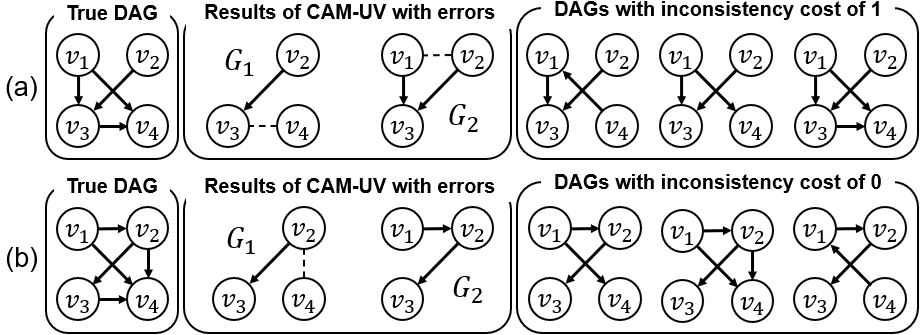}
\caption{
Examples of I-CAM-UV with relaxation.
}
\label{fig:realistic}
\end{figure}

Similar to Example \ref{ex:can-not-recover}, the ground truth DAG also cannot be recovered when CAM-UV estimates $(v_i, v_j) \in A_k$ where $(v_i, v_j) \notin A^*$.
Therefore, considering I-CAM-UV can assign directions to or exclude only edges in $E$, we need the following condition to recover the ground truth DAG:
% \begin{cond}
The input of I-CAM-UV satisfies $\hat{A} \subseteq A^*$, $A^* \setminus \hat{A} \subseteq \{(v_i,v_j) \mid \{v_i,v_j\} \in E\}$, and $C(G^*) \leq C^* + b$.
% \end{cond}

That is, it is desirable that CAM-UV estimation algorithms make no false discovery and capture a trace of each true causal relationship, even if it forms a falsely estimated UCP/UBP.
However, such desirable conditions are not always satisfied.
Therefore, in practical settings, we only require that CAM-UV algorithms yield an acyclic overlapped graph $\hat{G}$.
It is easily achieved by sharing estimated directed edges among all the datasets to avoid cycles.
Under the above setting, we handle Problem \ref{prb:relaxed} instead of Problem \ref{prb:ideal}.

%%%%%%%%%%%%%%%%%%%%%%%%%%%%%%%%%%%%%%%%%%%%%%%%%%
% Algorithm: I-CAM-UV
%%%%%%%%%%%%%%%%%%%%%%%%%%%%%%%%%%%%%%%%%%%%%%%%%%
\section{Algorithm}
% 概要
Now let us describe the proposed algorithm of I-CAM-UV to solve Problem \ref{prb:relaxed}.
I-CAM-UV essentially requires the consideration of an exponential number of graphs via all the connection patterns on $E$, i.e., we assign the direction to or exclude each $\{v_i, v_j\} \in E$, and there are $3^{|E|}$ such patterns.
Although the algorithm ignores cyclic graphs, there are still a huge number of DAGs to be considered.
Therefore, we introduce the key idea to reduce the computational cost by focusing on the monotonicity of the inconsistency cost.
That leads to designing a best-first search algorithm that enumerates DAGs in ascending order of the inconsistency cost.

% 全探索を考える
Let the variable pairs $E := \{e_1, e_2, \ldots, e_{|E|}\}$ be ordered.
We can exhaustively search all the DAGs by sequentially processing the variable pairs.
Let us denote a graph generated by adding a directed edge $(v_i,v_j)$ into $\tilde{G}$ as $\tilde{G}_{+(v_i,v_j)}$.
We manage each search state as a pair $(t, \tilde{G})$ where $t$ is the number of processed variable pairs and $\tilde{G}$ is a candidate solution graph.
First, we construct the overlapped DAG $\hat{G}$ and push the initial state $(0, \hat{G})$ into the top of the search queue.
After that, we iteratively pop a state $(t, \tilde{G})$ from the search queue to generate a finalized state $(|E|, \tilde{G})$ and the new states $(s+1, \tilde{G}_{+(v_i,v_j)})$ and $(s+1, \tilde{G}_{+(v_j,v_i)})$ for each $e_s = \{v_i, v_j\}$ where $s > t$.
We push the generated states into the top of the search queue, ignoring cyclic graphs.

% lower bound の導入
The above procedure seeks all the DAGs, not taking into account the inconsistency cost.
On the other hand, we are interested in only the graphs with the inconsistency costs of less than or equal to $C^* + b$.
Consequently, we have a chance to reduce the computation time by seeking limited DAGs.
Intuitively, a faster computation is achieved by popping states in ascending order of the minimum inconsistency cost of the successor DAGs until a popped state has a cost greater than $C^* + b$.
However, it is difficult to compute the minimum inconsistency cost of the successor DAGs.
Therefore, we use a lower bound of the inconsistency cost as an alternative.
%%%
\begin{dfn}
For any state $(t, \tilde{G})$ where $\tilde{G} = (\hat{V}, \hat{A} \cup \tilde{A})$, let $A_t := \bigcup_{e_s = \{v_i, v_j\} \in E, s > t}\{(v_i, v_j), (v_j, v_i)\}$ and $\tilde{G}[t] := (\hat{V}, \hat{A} \cup \tilde{A} \cup A_t)$.
That is, $\tilde{G}[t]$ has the bi-directed edge for each $e_s$ where $s > t$.
Note that UCP and UBP are well-defined even on cyclic graphs.
We define the cost function $\tilde{C}(t, \tilde{G}) := \sum_{k=1}^m (|\bar{I}_k(\tilde{G})| + |\bar{N}_k(\tilde{G}[t])|)$.
\end{dfn}
%%%
\begin{lem}
\label{lem:equality}
$C(\tilde{G}) = \tilde{C}(|E|, \tilde{G})$ holds for any state $(|E|, \tilde{G})$.
\end{lem}
%
%%%
\begin{thm}
\label{thm:monotonicity}
For any state $(t, \tilde{G})$ and any its successor state $(s, \tilde{G}')$, $\tilde{C}(t, \tilde{G}) \leq \tilde{C}(s, \tilde{G}')$ holds.
\end{thm}

According to Lemma \ref{lem:equality}, the cost function $\tilde{C}$ can naturally evaluate the inconsistency cost of finalized DAGs.
Moreover, according to Theorem \ref{thm:monotonicity}, we have the monotonicity of the cost function $\tilde{C}$.
We then propose to use the heap that prioritizes the states by ascending order of $\tilde{C}(t, \tilde{G})$.
As a result, the algorithm forms a best-first search, enumerates the DAGs in ascending order of the inconsistency cost, and is expected to reduce the computation time by seeking heuristically limited DAGs.
However, we have the remaining challenge to compute $\tilde{C}(t, \tilde{G})$.
It is equal to computing $\bar{I}_k(\tilde{G})$ and $\bar{N}_k(\tilde{G}[t])$ for all $k$ by searching UCPs/UBPs.
Fortunately, we can construct polynomial-time algorithms to search a UCP/UBP on a given graph as follows.

% メイン
\begin{algorithm}[t!]
\caption{I-CAM-UV}
\label{alg:main}
\begin{algorithmic}[1]
\STATE Compute $\hat{G}$ and $E\leftarrow E_\mathrm{imp} \cup E_\mathrm{uno}$
\STATE Set an edge order $E = \{e_1, e_2, \ldots, e_{|E|}\}$
\STATE Construct the heap $\mathcal{Q}$ with the priority $\tilde{C}$
\STATE Initialize $\mathcal{Q} \leftarrow \{(0, \hat{G})\}$, $\mathcal{G} \leftarrow \emptyset$, and $C^* \leftarrow \inf$
\WHILE{$\mathcal{Q}$ is not empty}
    \STATE Pop the top state $(t, \tilde{G})$ from $\mathcal{Q}$
    \IFLINE{$\tilde{C}(t, \tilde{G}) > C^* + b$}{break}
    \IF{$t = |E|$}
        \IFLINE{$C^* = \inf$}{$C^* \leftarrow C(\tilde{G})$}
        \STATE $\mathcal{G} \leftarrow \mathcal{G} \cup \{\tilde{G}\}$
    \ELSE
        \FOR{$s=t+1,\ldots,|E|$}
            \STATE $v_i, v_j \leftarrow e_s$
            \STATE $\tilde{G}_1 \leftarrow \tilde{G}_{+(v_i,v_j)}$ and $\tilde{G}_2 \leftarrow \tilde{G}_{+(v_j,v_i)}$
            \IFLINE{$\tilde{G}_1$ is a DAG}{Push $(s, \tilde{G}_1)$ into $\mathcal{Q}$}% with the priority $\tilde{C}(s, \tilde{G}_1)$}
            \IFLINE{$\tilde{G}_2$ is a DAG}{Push $(s, \tilde{G}_2)$ into $\mathcal{Q}$}% with the priority $\tilde{C}(s, \tilde{G}_2)$}
        \ENDFOR
        \STATE Push $(|E|, \tilde{G})$ into $\mathcal{Q}$% with the priority $\tilde{C}(|E|, \tilde{G})$
    \ENDIF
\ENDWHILE
\RETURN $\mathcal{G}$
\end{algorithmic}
\end{algorithm}

% UCPの探索
\paragraph{Search a UCP}
Given a graph $G = (\hat{V}, A)$ and observed variables $W \subset \hat{V}$, we show an algorithm to search a UCP of $\{v_i,v_j\} \subseteq W$ on $G$.
Let $T_i := \{v_k \in \hat{V} \setminus W \mid (v_k, v_i) \in A\}$ and $T_j := \{v_k \in \hat{V} \setminus W \mid (v_k, v_j) \in A\}$ be the unobserved variables of one hop before $v_i$ and $v_j$ respectively.
A path from $v_i$ (resp. $v_j$) to $v_k \in T_j$ not passing through $v_j$ (resp. $v_i$) makes a UCP $v_i \rightarrow \cdots \rightarrow v_k \rightarrow v_j$ (resp. $v_j \rightarrow \cdots \rightarrow v_k \rightarrow v_i$).
It is found in $O(|A|)$ time by a breadth-first search.

% UBPの探索
\paragraph{Search a UBP}
Given a graph $G = (\hat{V}, A)$ and observed variables $W \subset \hat{V}$, we show an algorithm to search a UBP of $\{v_i,v_j\} \subseteq W$ on $G$.
We define $T_i$ and $T_j$ as in the case of UCP.
Moreover, let $\mathbf{1}^G_{v_iv_j}[u, w]$ indicates whether $G$ has a path from $u$ to $w$ not passing through $v_i$ and $v_j$.
Let $S_i := \{v_a \in \hat{V} \mid \exists v_x \in T_i, \mathbf{1}^G_{v_iv_j}[v_a,v_x]\}$ and $S_j := \{v_a \in \hat{V} \mid \exists v_y \in T_j, \mathbf{1}^G_{v_iv_j}[v_a,v_y]\}$.
We have a UBP between $v_i$ and $v_j$ if and only if $S_i \cap S_j \neq \emptyset$, i.e., there is a path $v_i \leftarrow v_x \leftarrow \cdots \leftarrow v_a \rightarrow \cdots \rightarrow v_y \rightarrow v_j$ where $v_a \in S_i \cap S_j$ and $v_x, v_y \in \hat{V} \setminus W$.
We can compute $S_i$ and $S_j$ in $O(|A|)$ time by a breadth-first search.

% 疑似コードを示す
\paragraph{Pseudocode}
Algorithm \ref{alg:main} summarizes the best-first search procedure of I-CAM-UV.
The algorithm initializes the heap with the state consisting of the overlapped DAG (line 4), iteratively pops the top state (line 6), and checks the stop condition (line 7).
If the algorithm continues, it adds the DAG of the current state into the solution set if $t = |E|$ (line 10), otherwise pushes the next states into the heap, computing their priority (line 12-17).
Note that the algorithm computes $C^*$ when the first solution DAG is obtained (line 9).

% 計算量
\paragraph{Computational Complexity}
The proposed I-CAM-UV algorithm has computational complexity depending on the number of searched states.
The algorithm searches states in ascending order of $\tilde{C}$, stops when the popped state satisfies $\tilde{C}(t, \tilde{G}) > C^* + b$, and generates up to $2(|E| - t) + 1$ successors for each state.
Thus, let $Q$ be the number of the possible states with $\tilde{C}(t, \tilde{G}) \leq C^* + b$, and we have $O(Q|E|)$ searched states.
On computing $\tilde{C}$, because we need to search a UCP/UBP for each variable pairs on all the datasets, the algorithm takes $O((|\hat{A}| + |E|)\sum_{k=1}^m |V_k|^2)$ time for each state where $|\hat{A}| + |E|$ is the maximum possible graph size.
Therefore, let $R = (|\hat{A}| + |E|)\sum_{k=1}^m |V_k|^2$ for short, and the proposed best-first search algorithm via a heap takes $O(Q|E|(R + \log(Q|E|))$ time.

%%%%%%%%%%%%%%%%%%%%%%%%%%%%%%%%%%%%%%%%%%%%%%%%%%
% Experiments
%%%%%%%%%%%%%%%%%%%%%%%%%%%%%%%%%%%%%%%%%%%%%%%%%%
\section{Experiments}
We conducted experiments to demonstrate the usefulness of I-CAM-UV.
We posed the following four questions and answered each of them through experiments:
\begin{itemize}
\item Q1. How accurately does I-CAM-UV recover CAM-UV misses on observed variable pairs?
\item Q2. How accurately does I-CAM-UV discover causal relationships on unobserved variable pairs $E_\mathrm{uno}$.
\item Q3. How many DAGs does I-CAM-UV enumerate, and how does the accuracy distribute on the DAGs?
\item Q4. How realistically can I-CAM-UV operate in terms of computation times?
\end{itemize}

We implemented the core engine for the search procedure of I-CAM-UV in C++ and wrapped it with Cython.
The finalized module of I-CAM-UV and other experimental modules were implemented by Python.
The computer environment consisted of Ubuntu 24.04.1 LTS, Intel(R) Xeon(R) E-2174G CPU, and 64GB RAM.

%%%%%%%%%%%%%%%%%%%%%%%%%%%%%%
% Setup
%%%%%%%%%%%%%%%%%%%%%%%%%%%%%%
\subsection{Setup}
%% データセット
\paragraph{Datasets}
We generated 100 synthetic datasets following CAM based on random graphs of Erd\H{o}s–R\'{e}nyi model~\cite{Erdos-Renyi} with ten variables, the edge probability parameter $0.3$, and some non-linear functions.
The same setting is used in the existing CAM-UV literatures~\cite{Maeda21UAI,pham2025camuvx}.
From each dataset, we constructed a instance by randomly resampling datasets with non-identical variable sets where the number of datasets was $m \in \{2,3\}$ and the number of unobserved variables per dataset was $|U| \in \{3, 4\}$.
In each instance, at least $|U|$ variable pairs with causal relationships were simultaneously observed (resp. unobserved).
Moreover, each dataset was generated to have at least one common observed variable with one of the other datasets.
The number of observations sampled per dataset was 1,000.

%% 手法群
\paragraph{Competitors}
We set $b = 0$ for I-CAM-UV, i.e., we enumerated the DAGs with the minimum inconsistency cost.
In addition, we compared I-CAM-UV with the following methods:
(\emph{CAM-UV-OVL}) We used the simple overlapped DAG of CAM-UV results as a baseline.
(\emph{PC-OVL}) We applied the PC algorithm to each dataset and overlapped the estimated directed edges.
We took care to make the overlapped graph acyclic by ignoring later added edges if they make a cycle.
(\emph{Imputation}) Because unobserved variables can be treated as missing values, we combined the datasets into one by $k$-Nearest Neighbor imputation of $k=5$.
After the imputation, we applied CAM to the combined dataset with randomly resampling 1,000 observations to reduce the computational cost.
(\emph{CD-MiNi}) Although we focused on CAM, we directly applied CD-MiNi that estimates a LiNGAM on multiple datasets with non-identical variable sets.

% 評価指標
\paragraph{Metrics}
Since I-CAM-UV outputs multiple DAGs, basic metrics cannot be used directly.
Therefore, focusing on the frequency of edge existence on output DAGs, we modified true positive, false positive, and false negative, and called them MTP, MFP, and MFN.
For a graph set $\mathcal{G}$ and two different variables $v_i$ and $v_j$, let $c_{ij}$ be the number of graphs in $\mathcal{G}$ having the directed edge $(v_i,v_j)$.
With respect to the ground truth causal graph of the directed edge set $A^*$, we define $\mathrm{MTP} = \sum_{(v_i,v_j) \in A^*} \frac{c_{ij}}{|\mathcal{G}|}$, $\mathrm{MFP} = \sum_{(v_i,v_j) \notin A^*} \frac{c_{ij}}{|\mathcal{G}|}$, and $\mathrm{MFN} = \sum_{(v_i,v_j) \in A^*} \frac{|\mathcal{G}| - c_{ij}}{|\mathcal{G}|}$.
We then used the modified variations of recall $\frac{\mathrm{MTP}}{\mathrm{MTP} + \mathrm{MFN}}$, precision $\frac{\mathrm{MTP}}{\mathrm{MTP} + \mathrm{MFP}}$, and their harmonic mean (F1 score).
All the modified metrics are equal to the original ones when the graph set is a singleton.
Therefore, the modified metrics can be applied to all the competitors without any issues.

%%% accuracies
\begin{figure}[!t]
\centering
\includegraphics[width=0.99\linewidth]{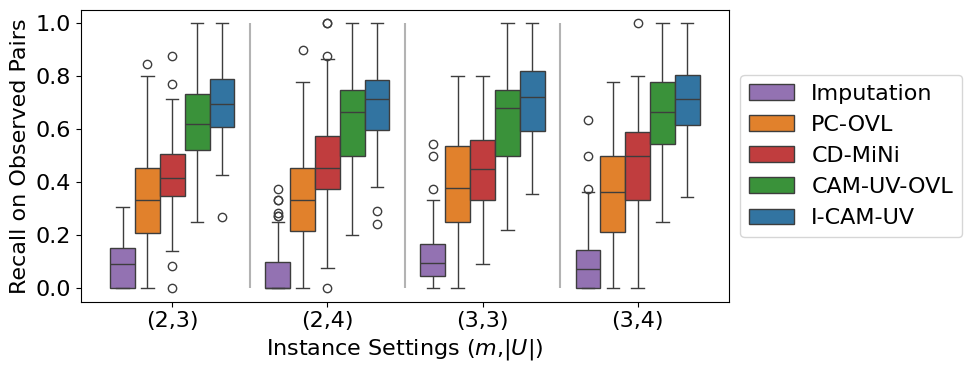}
\centering
\includegraphics[width=0.99\linewidth]{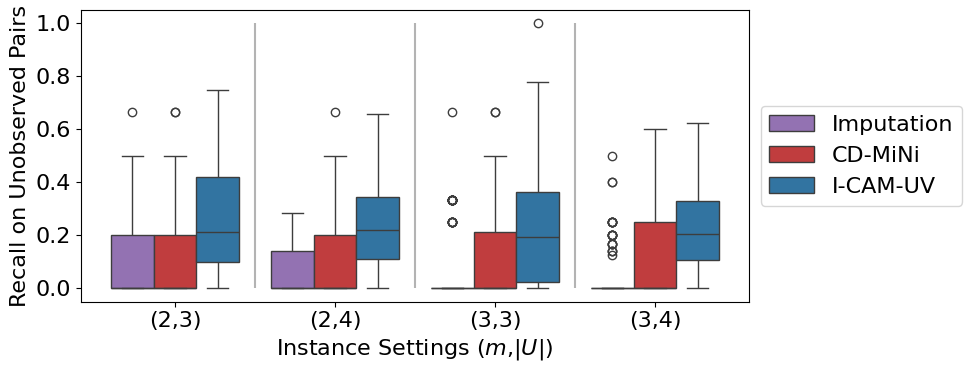}\\
\centering
\includegraphics[width=0.99\linewidth]{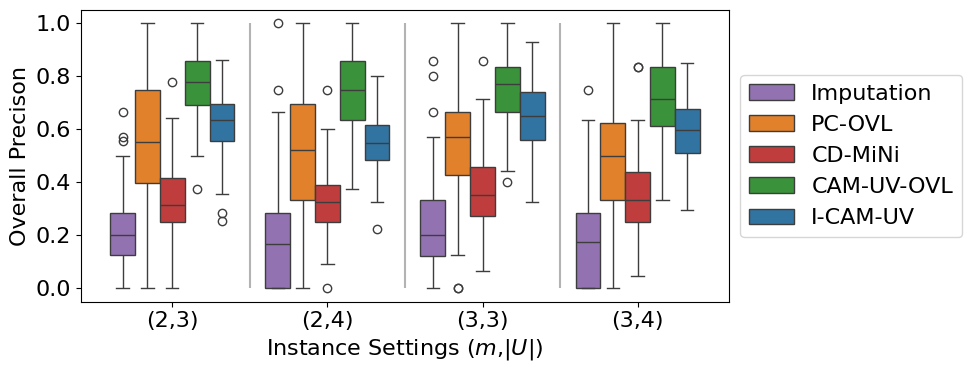}\\
\centering
\includegraphics[width=0.99\linewidth]{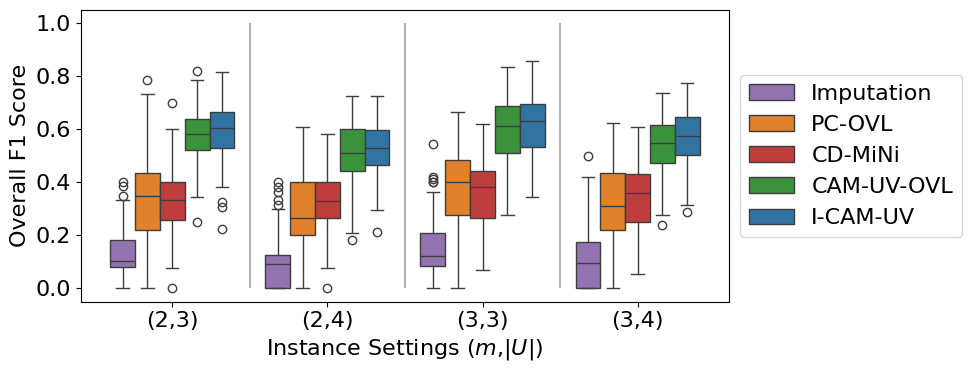}\\
\caption{Accuracies on synthetic datasets.}
\label{fig:accuracies}
\end{figure}

%%%%%%%%%%%%%%%%%%%%%%%%%%%%%%
% Results for Q1 and Q2
%%%%%%%%%%%%%%%%%%%%%%%%%%%%%%
\subsection{Results}% for Q1 and Q2}
\paragraph{Regarding Q1 and Q2}
To answer Q1 and Q2, we individually evaluated the recall scores on observed and unobserved variable pairs.
In addition, we also evaluated the overall precision and F1 scores.
Figure \ref{fig:accuracies} summarizes the results.

We found that the recall scores of I-CAM-UV were superior to all the competitors on both observed and unobserved variable pairs.
On the other hand, I-CAM-UV decreased the precision scores from that of CAM-UV-OVL.
Nonetheless, the overall F1 scores showed almost no difference between CAM-UV-OVL and I-CAM-UV.
While the results support the ability of I-CAM-UV to discover causal relationships, which is challenging to identify by only analyzing observed samples, the results also indicate that I-CAM-UV outputs a similar number of false discoveries.

Consequently, conclusions regarding Q1 and Q2 are as follows:
(A1 and A2) \emph{I-CAM-UV has superior recall compared with the other methods to recover causal relationships missed by CAM-UV and discover causal relationships on unobserved variable pairs. However, we need to take into account that the output DAGs may contain a similar number of false discoveries.}

%%%%%%%%%%%%%%%%%%%%%%%%%%%%%%
% Results for Q3
%%%%%%%%%%%%%%%%%%%%%%%%%%%%%%
\paragraph{Regarding Q3}
To answer Q3, we investigated the number of DAGs enumerated by I-CAM-UV and their individual accuracies.
The left one of Figure \ref{fig:num-and-cost} shows statistics on the number of DAGs enumerated by I-CAM-UV.
We selected some instances with a moderate number of output DAGs and plotted their accuracy distributions as shown in Figure \ref{fig:distribution}.

We found that the number of DAGs enumerated by I-CAM-UV could be both little and huge.
Since the behavior of I-CAM-UV depends on CAM-UV results, it is difficult to predict the output size.
Fortunately, the accuracy distribution seems to form a single cluster, that is, the enumerated DAGs are mostly biased toward similar accuracies.
These cluster-like accuracy distributions have been observed in many instances regardless the number of DAGs.
In addition, as a similar result of the consequences leading to A1 and A2, the accuracy of each DAG enumerated by I-CAM-UV is often superior to the competitors in terms of the recall.
Therefore, in practical applications, even if a huge number of DAGs are obtained by I-CAM-UV, a random sampling will narrow them down to a small alternative DAG set.

In conclusion regarding Q3, (A3) \emph{while a wide range of numbers of DAGs are output by I-CAM-UV, almost all DAGs achieve similar accuracy. A random sampling approach may be effective for practical analysis even if a huge number of DAGs are obtained.}

%%%%%%%%%%%%%%%%%%%%%%%%%%%%%%
% Results for Q4
%%%%%%%%%%%%%%%%%%%%%%%%%%%%%%
% \subsection{Results for Q4}
\paragraph{Regarding Q4}
To answer Q4, we report computation times of all the competitors on each setting in Figure \ref{fig:time}.
Note that the process of I-CAM-UV is only the enumeration of consistent graphs, and its former process obtaining CAM-UV results corresponds to CAM-UV-OVL.
Thus, the plots of I-CAM-UV in Figure \ref{fig:time} show the computation times of only the enumeration process.

We found that the enumeration process of I-CAM-UV was fast enough compared with CAM-UV-OVL in many instances, and the total computation times were comparable with the other methods.
That is, we confirmed that the best-first strategy of the I-CAM-UV algorithm is an efficient approach.
However, some instances take significantly more computation time as outliers.
Since the search space of I-CAM-UV depends on CAM-UV results, it is difficult to predict the computation time of I-CAM-UV.

In conclusion regarding Q4, (A4) \emph{I-CAM-UV runs in a realistic computation time assuming a relatively sparse DAG with ten variables. Since I-CAM-UV is an exponential time algorithm in general, we should be cautious when handling more variables. Nonetheless, because the worst case has $3^{|E|}$ search states, we can easily apply I-CAM-UV to small $E$ even if the number of variables is large.}

%%%% Number of Graphs and Minimum Costs
\begin{figure}[!t]
\centering
\includegraphics[width=0.99\linewidth]{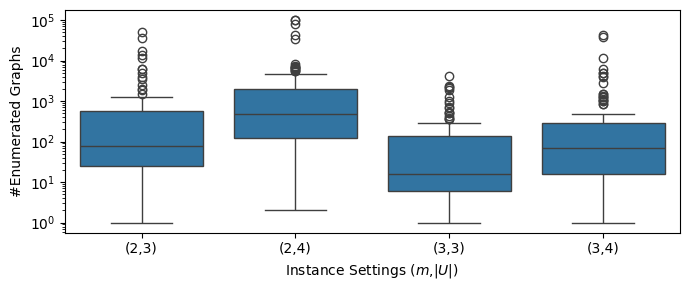}
\caption{Number of DAGs enumerated by I-CAM-UV.}
\label{fig:num-and-cost}
\end{figure}

%%% Accuracy Distributions
\begin{figure}[!t]
\centering
\includegraphics[width=0.99\linewidth]{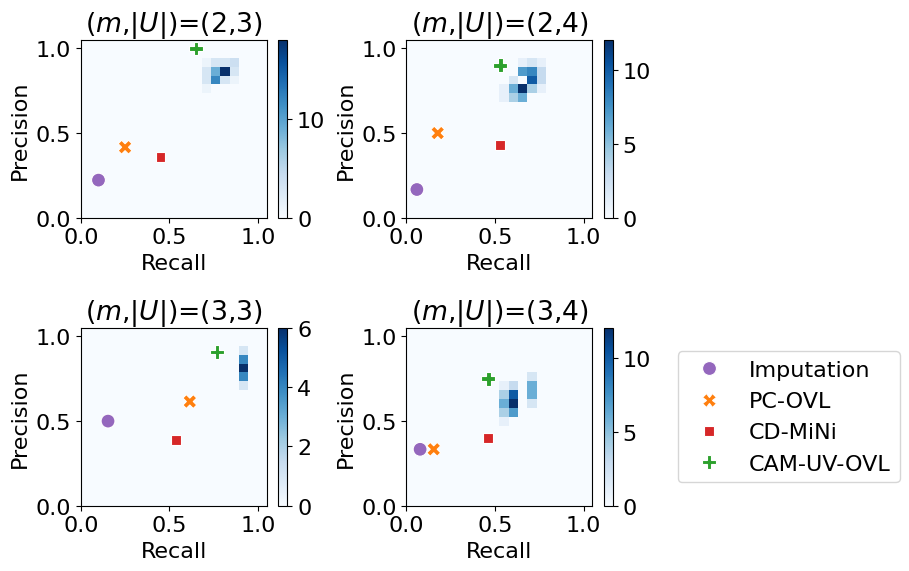}
\caption{
Accuracy distributions on specific instances of synthetic datasets.
The heatmaps show the distributions of DAGs output by I-CAM-UV.
}
\label{fig:distribution}
\end{figure}

%%% Computation Times
\begin{figure}[!t]
\centering
\includegraphics[width=0.99\linewidth]{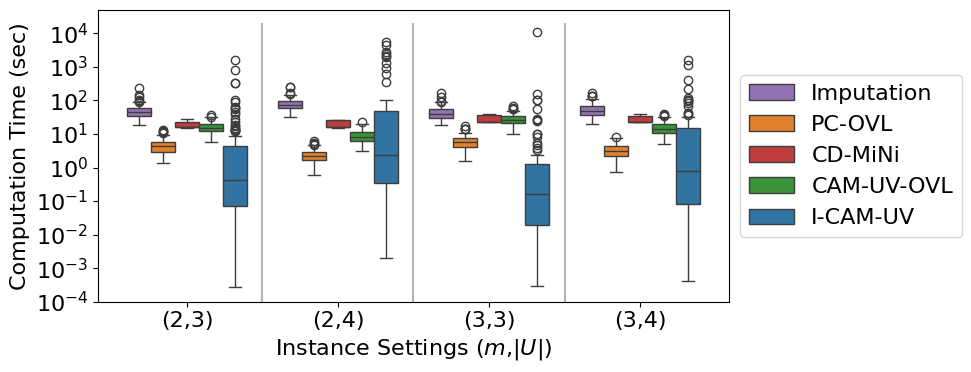}
\caption{Computation times on synthetic datasets.}
\label{fig:time}
\end{figure}

%%%%%%%%%%%%%%%%%%%%%%%%%%%%%%%%%%%%%%%%%%%%%%%%%%
% Conclusion and Discussion
%%%%%%%%%%%%%%%%%%%%%%%%%%%%%%%%%%%%%%%%%%%%%%%%%%
\section{Conclusion}
We proposed I-CAM-UV, enumerating integrated causal graphs as consistently as possible given CAM-UV results over non-identical variable sets.
The algorithm is an efficient combinatorial search leveraging the monotonicity of the inconsistency cost of DAGs.
We demonstrated that I-CAM-UV recovered causal relationships missed by CAM-UV, discovered causal relationships on unobserved variable pairs, and output consistent DAGs of similar accuracies.

\newcommand{\myparagraph}[1]{ \noindent {\mbox{\textbf{#1}~~}}}

% Limitation を書いておくとよい？
\myparagraph{Limitations}
We think that I-CAM-UV has the following three limitations.
(a) While tractable properties of UCP and UBP provide I-CAM-UV, extending the algorithm beyond CAM-UV is non-trivial.
(b) Because we may obtain a huge number of DAGs by I-CAM-UV, human checks are not easy for all the individual DAGs.
(c) The accuracy of I-CAM-UV is highly dependent on the accuracy of CAM-UV results.
Based on these considerations, future work is as follows.

\myparagraph{Future Work}
From a theoretical perspective, the following three points are of interest:
(i-a) Can we construct an algorithm integrating causal graphs with unobserved variables via other methods than CAM-UV, such as RCD~\citep{Maeda2020RCD}?
(i-b) How can we construct a Markov equivalence class-like notion for enumerated DAGs of I-CAM-UV?
(i-b') What are the conditions under which a unique DAG is obtained by I-CAM-UV?
On the other hand, from a practical perspective, the following three points are essential:
(ii-b) Constructing a compressed and interpretable representation of enumerated DAGs.
(ii-b') Developing an algorithm to evaluate the fitness of a given DAG over multiple datasets with non-identical variable sets.
(ii-c) Making I-CAM-UV more practical by improving the accuracy of CAM-UV itself.

%%%%%%%%%%%%%%%%%%%%%%%%%%%%%%
% Acknowledgements
%%%%%%%%%%%%%%%%%%%%%%%%%%%%%%
% TODO in Camera Ready

%%%%%%%%%%%%%%%%%%%%%%%%%%%%%%%%%%%%%%%%%%%%%%%%%%
% References
%%%%%%%%%%%%%%%%%%%%%%%%%%%%%%%%%%%%%%%%%%%%%%%%%%
\bibliography{aaai2026}

@misc{pham2025camuvx,
      title={Causal Additive Models with Unobserved Causal Paths and Backdoor Paths}, 
      author={Thong Pham and Takashi Nicholas Maeda and Shohei Shimizu},
      year={2025},
      eprint={2502.07646},
      archivePrefix={arXiv},
      primaryClass={cs.LG},
      url={https://arxiv.org/abs/2502.07646}, 
}

@book{Spirtes2000causation,
  title={Causation, prediction, and search},
  author={Spirtes, Peter and Glymour, Clark N and Scheines, Richard},
  year={2000},
  publisher={MIT press}
}

@article{Fu2025nature,
  title = {A foundation model of transcription across human cell types},
  author = {Fu, Xi and Mo, Shentong and Buendia, Alejandro and Laurent, Anouchka P. and Shao, Anqi and Alvarez‑Torres, Maria del Mar and Yu, Tianji and Tan, Jimin and Su, Jiayu and Sagatelian, Romella and Ferrando, Adolfo A. and Ciccia, Alberto and Lan, Yanyan and Owens, David M. and Palomero, Teresa and Xing, Eric P. and Rabadan, Raul},
  journal = {Nature},
  year = {2025},
  volume = {637},
  pages = {965--973},
  doi = {10.1038/s41586-024-08391-z},
  month = jan,
}

@article{Campomanes2014origin,
  title={Origin of the spectral shifts among the early intermediates of the rhodopsin photocycle},
  author={Campomanes, Pablo and Neri, Marilisa and Horta, Bruno AC and Röhrig, Ute F and Vanni, Stefano and Tavernelli, Ivano and Rothlisberger, Ursula},
  journal={Journal of the American Chemical Society},
  volume={136},
  number={10},
  pages={3842--3851},
  year={2014},
  publisher={ACS Publications}
}

@article{Fu2025Environmental,
title = {Integrating explainable AI and causal inference to unveil regional air quality drivers in China},
journal = {Journal of Environmental Management},
volume = {390},
pages = {126270},
year = {2025},
issn = {0301-4797},
doi = {https://doi.org/10.1016/j.jenvman.2025.126270},
author = {Zhiyuan Fu and Xiao Yang and Yike Ma and Yuhang Sun and Tianlian Wang},
}

@article{Smith2025Multi,
title = {Multi-scale causality in active matter},
journal = {Computers \& Chemical Engineering},
volume = {197},
pages = {109052},
year = {2025},
issn = {0098-1354},
doi = {https://doi.org/10.1016/j.compchemeng.2025.109052},
url = {https://www.sciencedirect.com/science/article/pii/S0098135425000560},
author = {Alexander Smith and Dipanjan Ghosh and Andrew Tan and Xiang Cheng and Prodromos Daoutidis},
}

@article{Runge2023causal,
  title={Causal inference for time series},
  author={Runge, Jakob and Gerhardus, Andreas and Varando, Gherardo and Eyring, Veronika and Camps-Valls, Gustau},
  journal={Nature Reviews Earth \& Environment},
  volume={4},
  number={7},
  pages={487--505},
  year={2023},
  publisher={Nature Publishing Group UK London}
}

@article{Zhou2025tcbbio,
  title = {Multi‑Objective Structure‑Based Drug Design Using Causal Discovery},
  author = {Zhou, Jingyuan and Zhao, Dengwei and Hao, Qian and Tu, Shikui and Xu, Lei},
  journal = {IEEE/ACM Transactions on Computational Biology and Bioinformatics},
  year = {2025},
  volume = {PP},
  number = {99},
  pages = {1--12},
  doi = {10.1109/TCBBIO.2025.3572178}
}

@incollection{Tillman2009ION,
  title={Integrating locally learned causal structures with overlapping variables},
  author={Tillman, Robert E. and Danks, David and Glymour, Clark},
  booktitle={Advances in Neural Information Processing Systems},
  pages={1665--1672},
  year={2009},
  publisher = {Curran Associates Inc.}
}

@article{shimizu2006,
	author = {Shimizu, Shohei and Hoyer, Patrik O. and Hyv{\"a}rinen, Aapo and Kerminen, Antti},
	journal = {Journal of Machine Learning Research},
	number = {Oct},
	pages = {2003--2030},
	title = {A linear {non-Gaussian} acyclic model for causal discovery},
	volume = {7},
	year = {2006}}

@article{chickering2002,
	Author = {Chickering, David Maxwell},
	Journal = {Journal of machine learning research},
	Number = {Nov},
	Pages = {507--554},
	Title = {Optimal structure identification with greedy search},
	Volume = {3},
	Year = {2002}}

@inproceedings{Triantafillou2010COmbINE,
  title={Learning causal structure from overlapping variable sets},
  author={Sofia Triantafillou and Ioannis Tsamardinos and Ioannis Tollis},
  booktitle={Proceedings of the Thirteenth International Conference on Artificial Intelligence and Statistics},
  pages={860--867},
  year={2010}
}

@inproceedings{Tillman2011IOD,
  title={Learning equivalence classes of acyclic models with latent and selection variables from multiple datasets with overlapping variables},
  author={Tillman, Robert E. and Spirtes, Peter},
  booktitle={Proceedings of the Fourteenth International Conference on Artificial Intelligence and Statistics},
  pages={3--15},
  year={2011}
}

@inproceedings{Huang2020CD-MiNi,
  title={Causal discovery from multiple data sets with non-identical variable sets},
  author={Huang, Biwei and Zhang, Kun and Gong, Mingming and Glymour, Clark},
  booktitle={Proceedings of the AAAI conference on artificial intelligence},
  volume={34},
  pages={10153--10161},
  year={2020}
}

@article{Buhlmann14CAM,
  title={{CAM}: Causal additive models, high-dimensional order search and penalized regression},
  author={B{\"u}hlmann, Peter and Peters, Jonas and Ernest, Jan},
  journal={Annals of Statistics},
  volume={42},
  number={6},
  pages={2526--2556},
  year={2014},
  publisher={Institute of Mathematical Statistics}
}

@inproceedings{Maeda21UAI,
  title={Causal additive models with unobserved variables},
  author={Maeda, Takashi Nicholas and Shimizu, Shohei},
  booktitle={Proc. 37th Conference on Uncertainty in Artificial Intelligence (UAI2021)},
  pages={97--106},
  year={2021},
  organization={PMLR}
}

@incollection{Hoyer09NIPS,
 title = {Nonlinear causal discovery with additive noise models},
 author = {Patrik O. Hoyer and Dominik Janzing and Joris Mooij and Jonas Peters and Bernhard Sch\"{o}lkopf},
 booktitle = {{Advances in Neural Information Processing Systems 21}},
 pages = {689--696},
 year = {2009},
 publisher = {Curran Associates Inc.}
}

@article{Tillman14BHMK,
  title={Learning causal structure from multiple datasets with similar variable sets},
  author={Tillman, Robert E. and Eberhardt, Frederick},
  journal={Behaviormetrika},
  volume={41},
  number={1},
  pages={41--64},
  year={2014}
}

@article{Mooij20JCI,
  author  = {Joris M. Mooij and Sara Magliacane and Tom Claassen},
  title   = {Joint Causal Inference from Multiple Contexts},
  journal = {Journal of Machine Learning Research},
  year    = {2020},
  volume  = {21},
  number  = {99},
  pages   = {1--108}
}

@article{Ding2019OICA,
  title={Likelihood-free overcomplete ICA and applications in causal discovery},
  author={Ding, Chenwei and Gong, Mingming and Zhang, Kun and Tao, Dacheng},
  journal={Advances in {N}eural {I}nformation {P}rocessing {S}ystems},
  volume={32},
  year={2019}
}

@article{Spirtes91PC,
	Author = {Peter Spirtes and Clark Glymour},
	Journal = {Social Science Computer Review},
	Pages = {67-72},
	Title = {An algorithm for fast recovery of sparse causal graphs},
	Volume = {9},
	Year = {1991}}

@inproceedings{Spirtes95FCI,
	Author = {Peter Spirtes and Christopher Meek and Thomas Richardson},
	Booktitle = {{Proc. 11th Annual Conference on Uncertainty in Artificial Intelligence (UAI1995)}},
	Title = {Causal Inference in the presence of latent variables and selection bias},
	Year = 1995,
	pages = {491-506}
	}

@article{Smith11NeuroImage,
	Author = {Stephen M. Smith and Karla L Miller and Gholamreza Salimi-Khorshidi and Matthew Webster and Christian F. Beckmann and Thomas E. Nichols and Joseph D. Ramsey and Mark W. Woolrich
},
	Journal = {NeuroImage},
	Pages = {875--891},
	Title = {Network modelling methods for {FMRI}},
	Volume = {54},
	number = {2},
	Year = {2011}}

@article{Shimizu11JMLR,
  title={{DirectLiNGAM}: A direct method for learning a linear non-{G}aussian structural equation model},
  author={Shimizu, Shohei and Inazumi, Takanori and Sogawa, Yasuhiro and Hyv{\"a}rinen, Aapo and Kawahara, Yoshinobu and Washio, Takashi and Hoyer, Patrik O and Bollen, Kenneth},
  journal={Journal of Machine Learning Research},
  volume={12},
  pages={1225--1248},
  year={2011},
  publisher={JMLR. org}
}

@InProceedings{Maeda2020RCD,
  title = 	 {{RCD}: Repetitive causal discovery of linear non-{G}aussian acyclic models with latent confounders},
  author =       {Maeda, Takashi Nicholas and Shimizu, Shohei},
  booktitle = 	 {Proc. 23rd International Conference on Artificial Intelligence and Statistics (AISTATS2010)},
  pages = 	 {735--745},
  year = 	 {2020},
  volume = 	 {108},
  series = 	 {Proceedings of Machine Learning Research},
  month = 	 {26--28 Aug},
  publisher =    {PMLR}
}

@article{Erdos-Renyi,
  author = {Erdos, Paul and Renyi, Alfred},
  title = {On the evolution of random graphs},
  journal = {Publication of the Mathematical Institute of the Hungarian Academy of Sciences},
  pages = {17-61},
  volume = 5,
  year = 1960
}

@misc{duncan1972,
author = {Duncan, Otis Dudley and Featherman, David L. and Duncan, Beverly},
title ={Socioeconomic background and achievement},
publisher = {New York, Seminar Press},
year = {1972}
}

@Article{Strobl2019RCIT,
journal={Journal of Causal Inference},
author={Strobl Eric V. and Zhang Kun and Visweswaran Shyam},
title={Approximate Kernel-Based Conditional Independence Tests for Fast Non-Parametric Causal Discovery},
year={2019},
month={March},
pages={1-24},
volume={7},
number={1},
doi={10.1515/jci-2018-0017}
}

@inbook{PyTorch,
author = {Paszke, Adam and Gross, Sam and Massa, Francisco and Lerer, Adam and Bradbury, James and Chanan, Gregory and Killeen, Trevor and Lin, Zeming and Gimelshein, Natalia and Antiga, Luca and Desmaison, Alban and K\"{o}pf, Andreas and Yang, Edward and DeVito, Zach and Raison, Martin and Tejani, Alykhan and Chilamkurthy, Sasank and Steiner, Benoit and Fang, Lu and Bai, Junjie and Chintala, Soumith},
title = {PyTorch: an imperative style, high-performance deep learning library},
year = {2019},
publisher = {Curran Associates Inc.},
address = {Red Hook, NY, USA},
booktitle = {Proceedings of the 33rd International Conference on Neural Information Processing Systems},
articleno = {721},
numpages = {12},
pages = {8026-8037}
}

%%%%%%%%%%%%%%%%%%%%%%%%%%%%%%%%%%%%%%%%%%%%%%%%%%
% Appendix / Supplementary Material
%%%%%%%%%%%%%%%%%%%%%%%%%%%%%%%%%%%%%%%%%%%%%%%%%%
\appendix

\section{Proofs}

\paragraph{Proof of Theorem \ref{thm:main-thm}}
\begin{proof}
If CAM-UV results have no estimation error, we immediately have the following four statements:
(i) $\hat{A} \subseteq A^*$ holds;
(ii) $\{v_i, v_j\} \in E$ holds for any $(v_i,v_j) \in A^* \setminus \hat{A}$;
(iii) $\mathrm{UP}_{G^*,V_k} = \emptyset$ holds for any $\{v_i, v_j\} \in I_k$;
(iv) $\mathrm{UP}_{G^*,V_k} \neq \emptyset$ holds for any $\{v_i, v_j\} \in N_k$.
Therefore, because we have $V = \bigcup_{k=1}^m V_k$ and $G^*$ can be obtained from $\hat{G} = (V, \hat{A})$ by assigning directions to or excluding edges in $E$, we can state $G^*$ is consistent according to Definition \ref{dfn:consistent}.
\end{proof}

\paragraph{Proof of Lemma \ref{lem:equality}}
\begin{proof}
Because $A_{|E|} = \emptyset$, we have $\tilde{G} = \tilde{G}[|E|]$ that implies $\bar{N}_k(\tilde{G}) = \bar{N}_k(\tilde{G}[|E|])$.
Hence, the lemma follows.
\end{proof}

\paragraph{Proof of Theorem \ref{thm:monotonicity}}
\begin{proof}
Because any UCP/UBP is lost only by specific edge removals, $\mathrm{UP}_{\tilde{G},V_k}(v_i, v_j) \subseteq \mathrm{UP}_{\tilde{G}',V_k}(v_i,v_j)$ holds for any $k$ and $\{v_i,v_j\} \in I_k$.
Similarly, because any UCP/UBP is not caused by any edge removal, $\mathrm{UP}_{\tilde{G}'[s],V_k}(v_i, v_j) \subseteq \mathrm{UP}_{\tilde{G}[t],V_k}(v_i,v_j)$ holds for any $k$ and $\{v_i,v_j\} \in N_k$.
Hence, we have $\bar{I}_k(\tilde{G}) \subseteq \bar{I}_k(\tilde{G}')$ and $\bar{N}_k(\tilde{G}[t]) \subseteq \bar{N}_k(\tilde{G}'[s])$ for any $k$, and the theorem follows.
\end{proof}

\section{Detailed Experimental Settings}

\subsection{Synthetic Dataset Generation}
We generated synthetic datasets on CAM with 10 variables as follows.
Skeletons of the ground truth causal graphs were based on Erd\H{o}s–R\'{e}nyi random graph model~\cite{Erdos-Renyi} with the edge generation probability 0.3, i.e., each variable pair was connected by an edge with the probability of 0.3.
To make the graph DAG, we randomly generated a permutation of 10 variables and assigned a direction to each edge according to the variable order in the generated permutation.
The value of each variable $v_i$ was determined by the following formulas.
\begin{align*}
v_i = \frac{h_i}{\mathrm{sd}(h_i)},~h_i = \sum_{v_j \in P_i^*} ((v_j + a_{i,j})^{c_{i,j}} + b_{i,j}) + n_i
\end{align*}
where $a_{i,j}$, $b_{i,j}$, and $c_{i,j}$ denote constants, $n_i$ denotes a random noise, and $\mathrm{sd}(h_i)$ denotes the standard deviation of $h_i$.
For more details, we randomly chosen $a_{i,j} \sim U(-5,5)$, $b_{i,j} \sim U(-1, 1)$, $c_{i,j} \in \{2,3\}$, and $n_i \sim U(-10 + d_i, 10 + d_i)$ where $d_i \sim U(-2, 2)$.
For the given number of observations, we sampled each variable in the topological order of the causal graph.
The above setting is used in existing CAM-UV literatures~\cite{Maeda21UAI,pham2025camuvx}.

\subsection{fMRI Dataset}
As an additional experimental evaluation of I-CAM-UV, we tried a well-known semi-real simulated dataset based on a mathematical model of interactions among brain regions.
The datasets were generated by~\cite{Smith11NeuroImage} and we picked one of them (``sim2''), which has 10 variables and 10,000 observations.

\begin{figure}[b]
\centering
\includegraphics[width=0.99\linewidth]{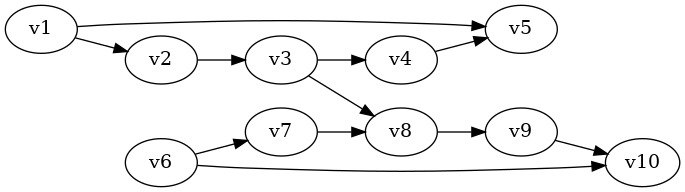}
\caption{Causal graph of fMRI dataset}
\label{fig:fmri-dag}
\end{figure}

We generated 100 random instances on fMRI dataset.
For each instance, we randomly resampled datasets with non-identical variable sets where the number of datasets was $m \in \{2,3\}$ and the number of unobserved variables per dataset was $|U| \in \{3, 4\}$.
We also took care to obtain at least $|U|$ observed (resp. unobserved) variable pairs having causal relationships.
Moreover, each dataset was generated to have at least one common variable with one of the other datasets.
The number of observations sampled per dataset was 1,000.

\subsection{Parameters of Competitors}
\begin{itemize}
\item I-CAM-UV: We set $b = 0$ for synthetic datasets and $b = 2$ for fMRI dataset.
For the pre-processing of I-CAM-UV to compute CAM-UV results, we used the same parameters as CAM-UV-OVL.
\item CAM-UV-OVL: We need to set the threshold $\alpha$ of independence tests, the number of explanatory variables $\mu_{\mathrm{base}}$ for the base search~\cite{Maeda21UAI}, and the number of explanatory variables $\mu_{\mathrm{ext}}$ for the extended search~\cite{pham2025camuvx}.
We set $\alpha = 0.05$, $\mu_{\mathrm{base}} = 2$, and $\mu_{\mathrm{base}} = 1$.
\item PC-OVL: We set the threshold of independence tests $\alpha = 0.05$, and we used Randomized Conditional Independence Test (RCIT)~\cite{Strobl2019RCIT} with a fixed random seed for fast computations on non-linear continuous relationships.
\item Imputation: We used the $k$-nearest neighbor-based imputation strategy of $k = 5$.
After the imputation, we estimated CAM by the same parameters as CAM-UV-OVL.
\item CD-MiNi: The algorithm~\cite{Huang2020CD-MiNi} is based on an overcomplete independent component analysis method~\cite{Ding2019OICA} by using Gaussian mixture models of noise variables and using a stochastic gradient descent.
We set the number of Gaussian components to $10$ and the sparse regularization parameter to $0.01$.
We used the default Adam optimizer of PyTorch~\cite{PyTorch} with the learning rate of $0.01$, the batch size of $100$, and $100$ epochs.
After the coefficient matrix of the estimated causal graph is obtained, we pruned the edges under the absolute coefficient threshold of $0.1$.
\end{itemize}

\section{Experimental Results on fMRI dataset}
We conclude the experimental results on the fMRI dataset as follows.
Figure \ref{fig:accuracies-fmri} shows the accuracies of the competitors, including recall scores on observed (resp. unobserved) variable pairs, overall precision scores, and overall F1 scores.
Figure \ref{fig:nsol-fmri} shows the distributions of the number of enumerated DAGs by I-CAM-UV.
Figure \ref{fig:distribution-fmri} shows the accuracy distributions of some specific instances.
Figure \ref{fig:time-fmri} shows the computation times of the competitors.

In the aspect of accuracy, we obtained the results similar to those of synthetic datasets, except for the recall scores on unobserved variable pairs.
On the fMRI dataset, we found that the imputation method could often achieve higher recall scores on unobserved variable pairs compared with DAGs of the minimum inconsistency cost ($b = 0$).
Thus, we investigated the accuracies of DAGs with the additional inconsistency costs ($b \in \{1,2\}$).
We then found that the recall scores of I-CAM-UV were improved on average as the cost increased.
On the other hand, although the precision scores of I-CAM-UV decreased as the cost increased, the F1 scores of I-CAM-UV had little or no changes.
Therefore, as with the experiments on synthetic datasets, while I-CAM-UV could recover/discover some correct causal relationships, there were a similar number of false discoveries.

The number of enumerated DAGs by I-CAM-UV was distributed from little to huge, and seemed to increase exponentially as the cost increased.
However, the accuracy distribution of enumerated DAGs on almost all instances made a single cluster.
Such properties on the distributions of the number of DAGs and their accuracy can be said to be similar to those of synthetic datasets.
On the other hand, we found that I-CAM-UV could run in a much faster time than the competitors.
This is likely due to the sparsity and characteristic shape of the causal graph assumed in the fMRI dataset.

%%% Accuracies
\begin{figure}
\centering
\includegraphics[width=0.99\linewidth]{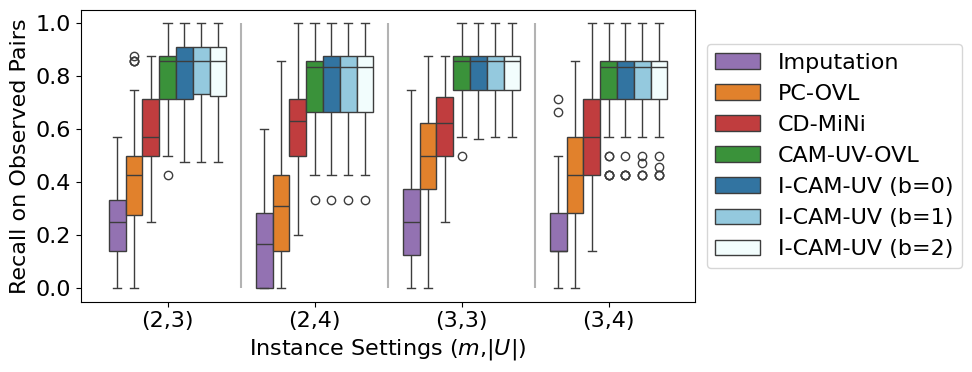}
\centering
\includegraphics[width=0.99\linewidth]{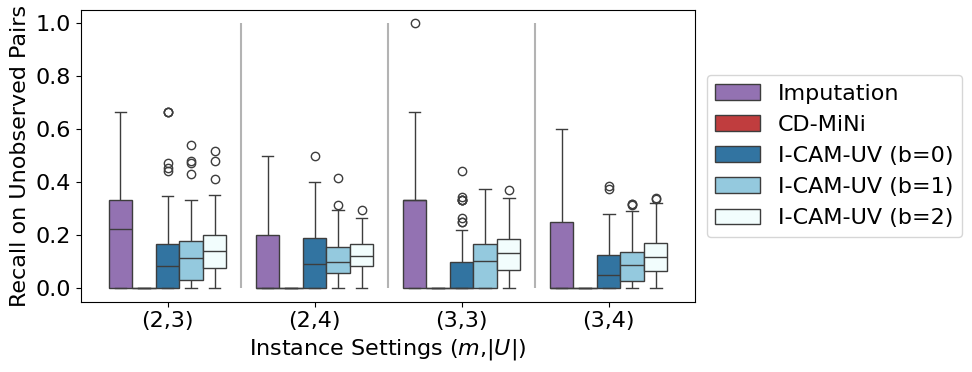}
\centering
\includegraphics[width=0.99\linewidth]{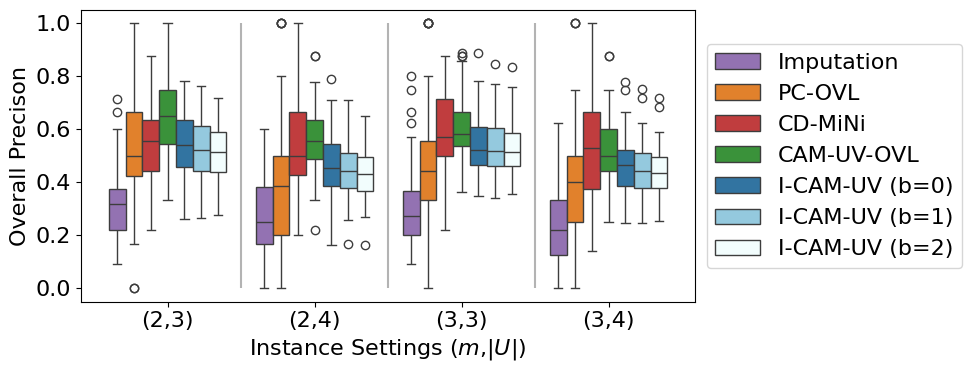}
\centering
\includegraphics[width=0.99\linewidth]{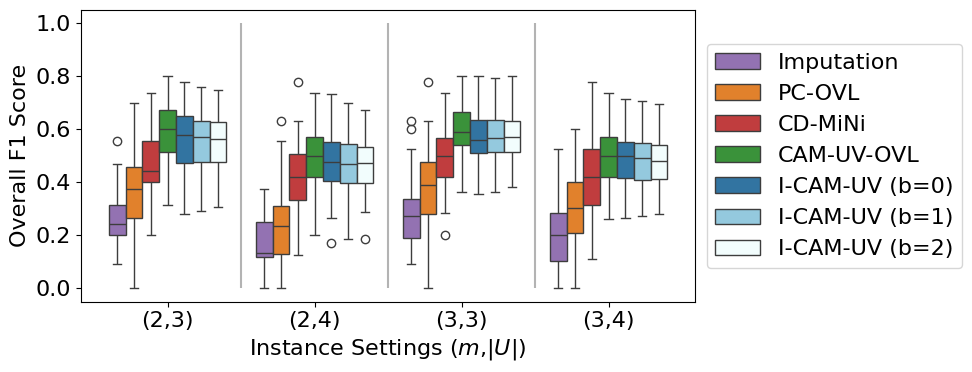}
\caption{Accuracies on fMRI dataset.}
\label{fig:accuracies-fmri}
\end{figure}

%%%% Number of Graphs and Minimum Costs
\begin{figure}[!t]
\centering
\includegraphics[width=0.99\linewidth]{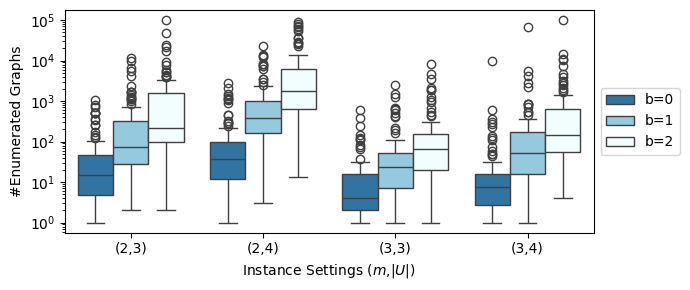}
\caption{Number of DAGs enumerated by I-CAM-UV on fMRI dataset.}
\label{fig:nsol-fmri}
\end{figure}

%%% Accuracy Distributions
\begin{figure}[!t]
\centering
\includegraphics[width=0.99\linewidth]{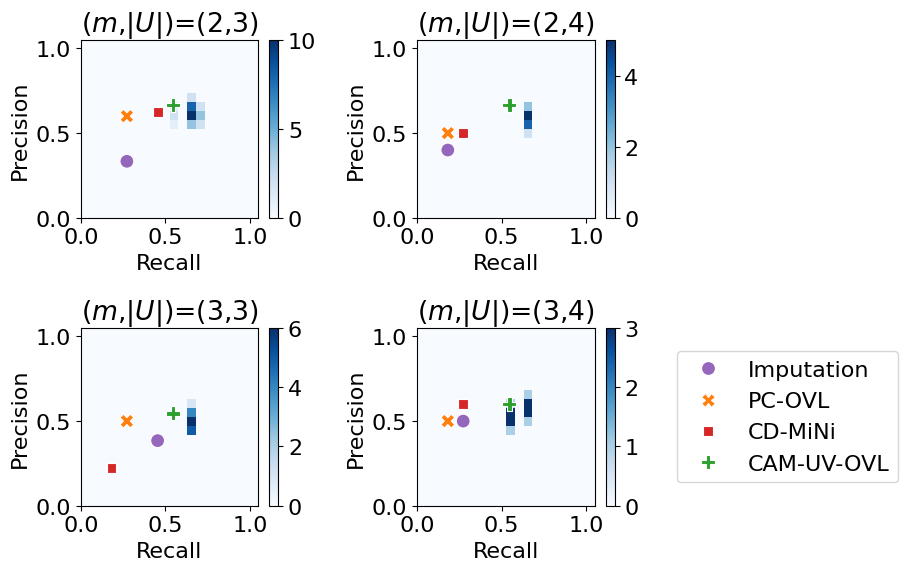}
\centering
\includegraphics[width=0.99\linewidth]{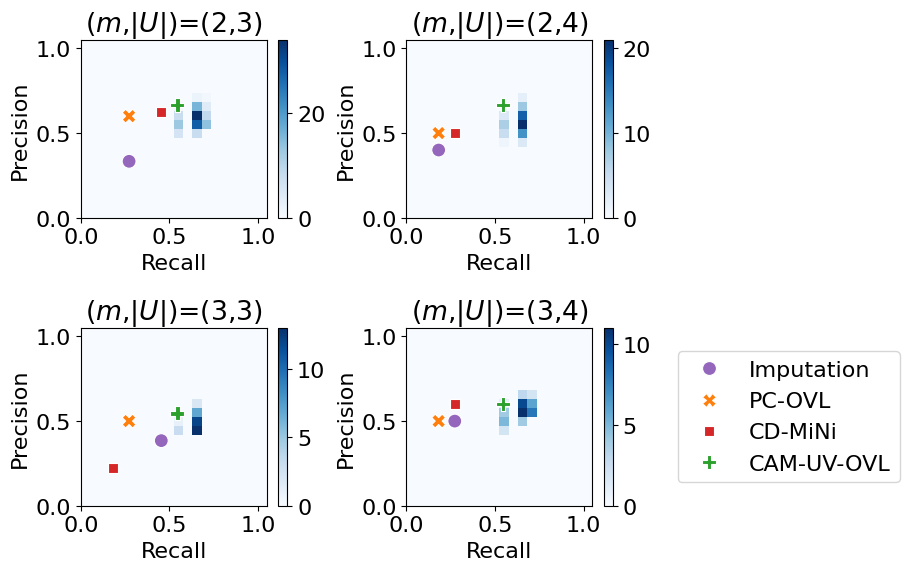}
\centering
\includegraphics[width=0.99\linewidth]{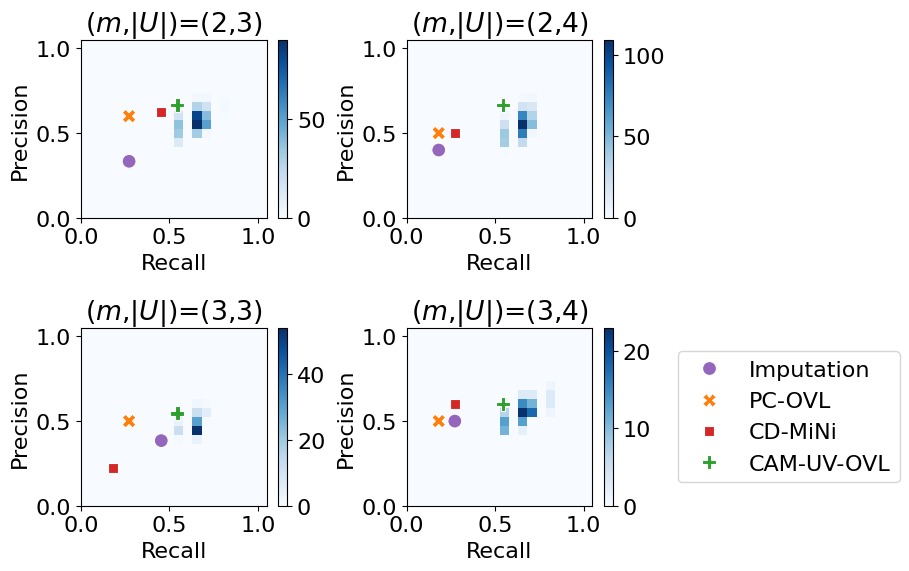}
\caption{
Accuracy distributions on specific instances of fMRI dataset.
The heatmaps show the distributions of DAGs output by I-CAM-UV.
The top fours are cases of $b = 0$, the middle fours are cases of $b = 1$, and the bottom fours are cases of $b = 2$.
}
\label{fig:distribution-fmri}
\end{figure}

%%% Computation Times
\begin{figure}[!t]
\centering
\includegraphics[width=0.99\linewidth]{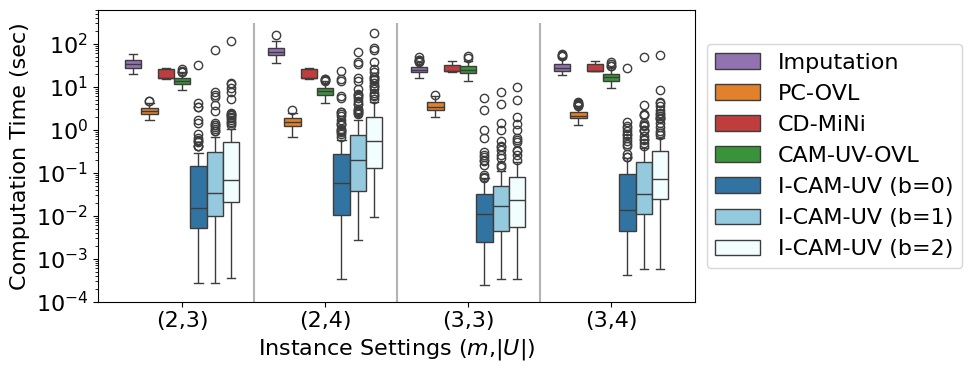}
\caption{Computation times on fMRI dataset.}
\label{fig:time-fmri}
\end{figure}

\section{Discussion on Minimum Inconsistency Cost}
I-CAM-UV enumerates DAGs of the minimum inconsistency cost for given CAM-UV results.
The inconsistency cost means how many variable pairs are inconsistent with the CAM-UV results in terms of the presence/absence of a UCP/UBP.
In other words, if the minimum cost is large, it suggests that there is no sufficiently appropriate DAG, and there may be an issue with the accuracy of CAM-UV.
On the other hand, it should be noted that, as discussed in the main paper, there are cases where the minimum cost becomes zero even if there are errors in the CAM-UV results.

Figure \ref{fig:cost-synthetic} and Figure \ref{fig:cost-fmri} show the distributions of minimum inconsistency costs on synthetic and fMRI datasets, respectively.
In our experimental settings, if the instance parameter is $(m,|U|)$, the \emph{maximum} inconsistency cost of the instance is $m \cdot \frac{(10 - |U|)(9 - |U|)}{2}$, i.e., $42$ for $(2,3)$, $30$ for $(2,4)$, $63$ for $(3,3)$, and $45$ for $(3,4)$.
As a result, we found that, in many instances, the minimum inconsistency cost was greater than one-tenth and less than one-third of the maximum inconsistency cost.
We consider the results indicate relatively high inconsistency costs, and believe that improving the accuracy of CAM-UV is an urgent task.

%%% Minimum Inconsistency Cost
\begin{figure}[!t]
\centering
\includegraphics[width=0.99\linewidth]{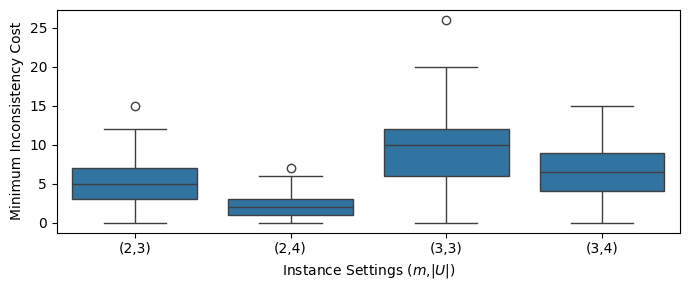}
\caption{Minimum inconsistency costs on synthetic dataset.}
\label{fig:cost-synthetic}
\end{figure}

\begin{figure}[!t]
\centering
\includegraphics[width=0.99\linewidth]{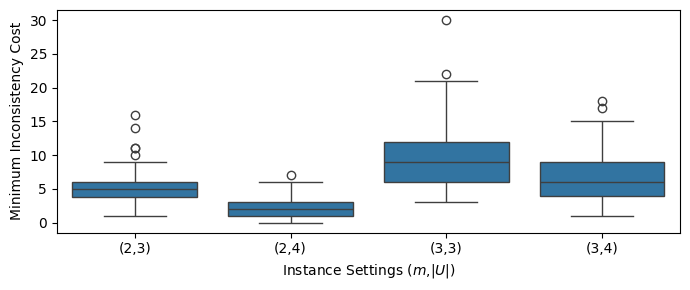}
\caption{Minimum inconsistency costs on fMRI dataset.}
\label{fig:cost-fmri}
\end{figure}

\section{Demo on Social Analysis Example}
We analyzed a social analysis dataset named State Attainment Model~\cite{Shimizu11JMLR}, which had been taken from a sociological data repository called General Social Survey~\footnote{http://www.norc.org/GSS+Website/}.
(This dataset is not publicly available. Please contact the authors of~\cite{Shimizu11JMLR} if needed.)
The dataset contained 1,380 samples and six observed variables: ``Father's Occupation'', ``Father's Education'', ``Son's Occupation'', ``Son's Education'', ``Number of Siblings'', and ``Son's Income''.
The sample selection had been conducted based on the following criteria: i) non-farm background; ii) ages 35 to 44; iii) white; iv) male; v) in the labor force at the time of the survey; vi) not missing data for any of the covariates; vii) years 1972-2006.
Figure \ref{fig:knowledge} shows domain knowledge about their causal relationships~\cite{duncan1972}.
As shown in the figure, there could be some latent confounders that violated the ideal situation of I-CAM-UV.
An objective of this example is to demonstrate how we can use I-CAM-UV in practical analysis.

We randomly divided the sample set into two equal-sized sample sets.
In the first dataset, we took only three variables ``Father's Occupation'', ``Father's Education'', and Son's Education.
This can be understood as a small dataset created to analyze how a son's educational level is influenced by his father.
In the second dataset, we took only four variables ``Son's Occupation'', ``Son's Education'', ``Number of Siblings'', and ``Son's Income''.
This can be understood as a small dataset created to analyze how an individual's income is influenced by their siblings and skills.
Our goal was to analyze how son's income has been influenced by his family environment and skills, integrating the CAM-UV result of the above two datasets.

First, we removed outliers from the two datasets by applying the IQR (Inter Quartile Range) method to each variable.
As a result, we obtained 683 and 603 samples for the first dataset and the second dataset, respectively.
In these datasets, because the reliability of independence tests is not enough due to the small samples, we set the relatively high threshold $\alpha = 0.3$ for CAM-UV.

After that, we obtained the CAM-UV results as shown in Figure \ref{fig:sam-cam-uv}.
These provide only two identified causal relationships from ``Son's Education'' to ``Son's Income'' and from ``Son's Occupation'' to ``Son's Income''.
We considered the situation that CAM-UV results could have some estimation errors.
Therefore, we set the I-CAM-UV parameter $b = 2$.
As a result, the minimum inconsistency cost was zero, and I-CAM-UV enumerated 10,483 DAGs with inconsistency costs less than or equal to two.

Because the target variable of this analysis is ``Son's Income'', we set a condition that ``Son's Income'' becomes the sink, for narrowing down enumerated DAGs.
As a result, the number of DAGs was reduced to 4,051.
We show three randomly sampled examples in Figure \ref{fig:sam-icamuv}.
We then found an intuitive causal relationship from ``Father's Education'' to ``Son's Education''.

Moreover, we add this intuitive causal relationship from ``Father's Education'' to ``Son's Education'' into the sampling conditions.
As a result, the number of DAGs was reduced to 1,575.
We show three randomly sampled examples in Figure \ref{fig:sam-icamuv+}.
We then found that there may be a causal relationship between ``Father's Occupation'' and ``Son's Education''.

On the above newly indicated causal relationship, the direction from ``Son's Education'' to ``Father's Occupation'' seemed to be contradictory.
Therefore, we set a further condition that there is a causal relationship from ``Father's Education'' to ``Son's Education''.
As a result, the number of DAGs was reduced to 623.
We show three randomly sampled examples in Figure \ref{fig:sam-icamuv++}.
We then obtained highly intuitive causal graphs.

%%% knowledge
\begin{figure*}
\centering
\includegraphics[width=0.99\linewidth]{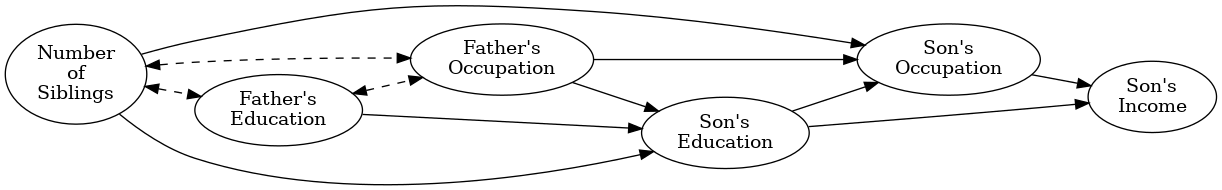}
\caption{
Status attainment model based on domain knowledge~\cite{duncan1972}.
The solid directed edges represent expected causal relationships.
The dashed bi-directed edges represent not modeled causal relationships, e.g., there may be latent confounders between the two, there may be a directed edge between the two, or the two may be independent.
}
\label{fig:knowledge}
\end{figure*}

%%% CAM-UV results
\begin{figure*}
%%%%% 0
\begin{minipage}[b]{0.49\linewidth}
\centering
\includegraphics[width=0.99\linewidth]{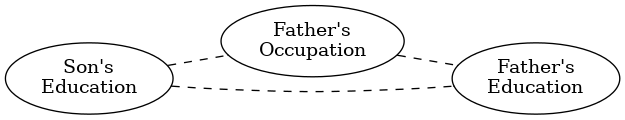}
\end{minipage}
%%%%% 1
\begin{minipage}[b]{0.49\linewidth}
\centering
\includegraphics[width=0.99\linewidth]{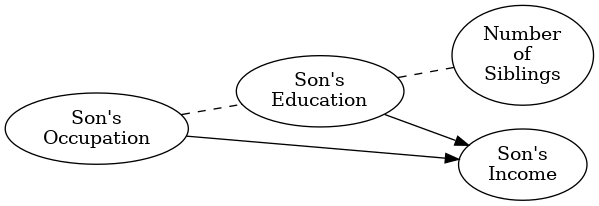}
\end{minipage}
%%%%%
\caption{
CAM-UV results for status attainment model dataset.
The solid directed edges represent estimated causal relationships.
The dashed edges represent not identified causal relationships due to the estimated existence of UCBs/UBPs.
}
\label{fig:sam-cam-uv}
\end{figure*}

%%% I-CAM-UV
\begin{figure*}
%%%%% 0
\centering
\includegraphics[width=0.99\linewidth]{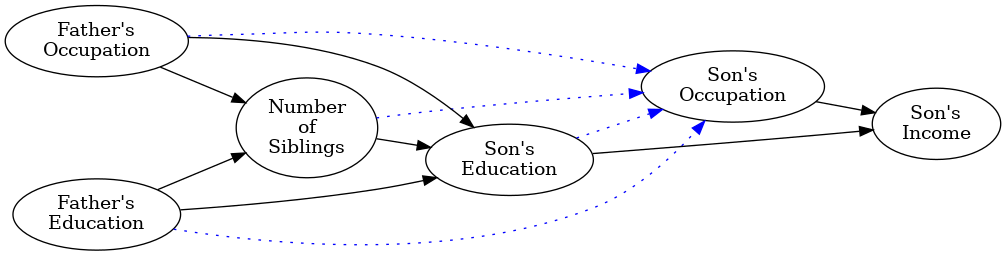}
%%%%% 1
\centering
\includegraphics[width=0.99\linewidth]{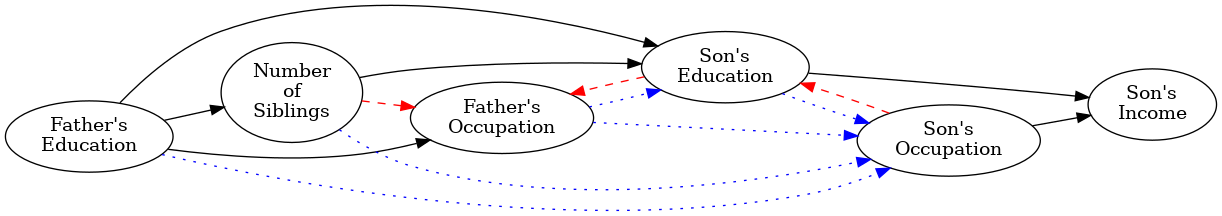}
%%%%% 2
\centering
\includegraphics[width=0.99\linewidth]{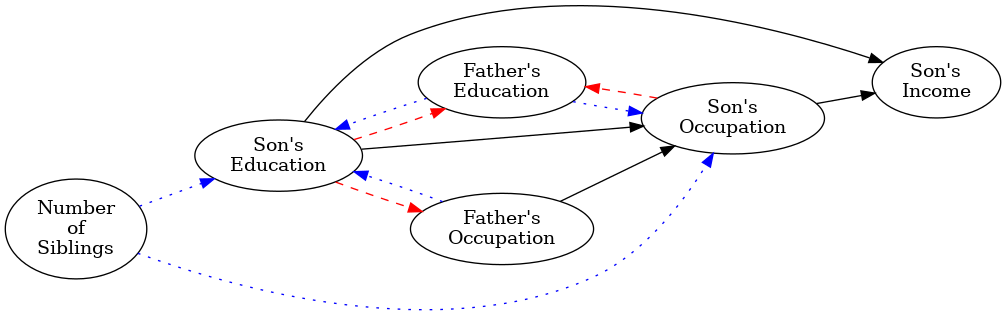}
%%%%%
\caption{
Three examples of the graphs enumerated by I-CAM-UV with specifying ``Son's Income'' becomes the sink.
The solid black directed edges represent correctly estimated causal relationships.
The dashed red directed edges represent wrongly estimated causal relationships.
The dotted blue directed edges represent expected but not estimated causal relationships.
}
\label{fig:sam-icamuv}
\end{figure*}

%%% I-CAM-UV+
\begin{figure*}
%%%%% 0
\centering
\includegraphics[width=0.99\linewidth]{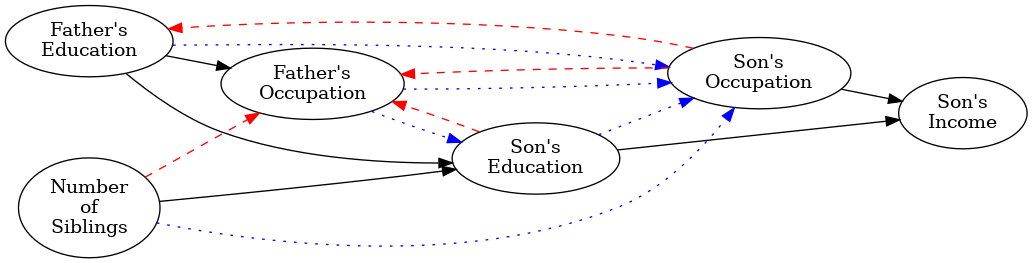}
%%%%% 1
\centering
\includegraphics[width=0.99\linewidth]{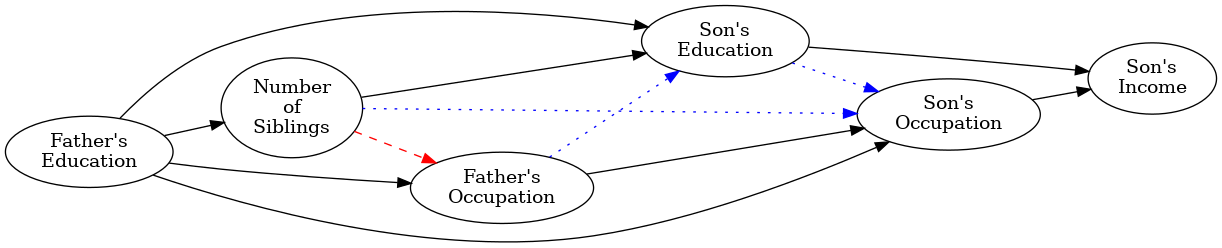}
%%%%% 2
\centering
\includegraphics[width=0.99\linewidth]{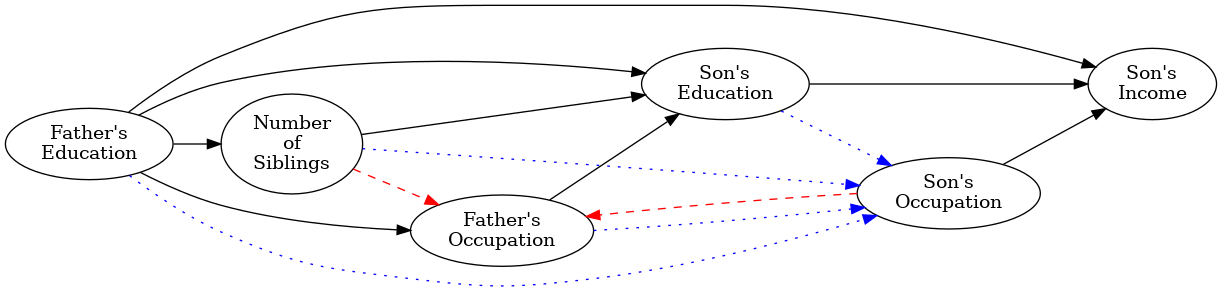}
%%%%%
\caption{
Three examples of the graphs enumerated by I-CAM-UV with specifying the following conditions: ``Son's Income'' becomes the sink, and there is a causal relationship from ``Father's Education'' to ``Son's Education''.
The solid black directed edges represent correctly estimated causal relationships.
The dashed red directed edges represent wrongly estimated causal relationships.
The dotted blue directed edges represent expected but not estimated causal relationships.
}
\label{fig:sam-icamuv+}
\end{figure*}

%%% I-CAM-UV++
\begin{figure*}
%%%%% 0
\centering
\includegraphics[width=0.99\linewidth]{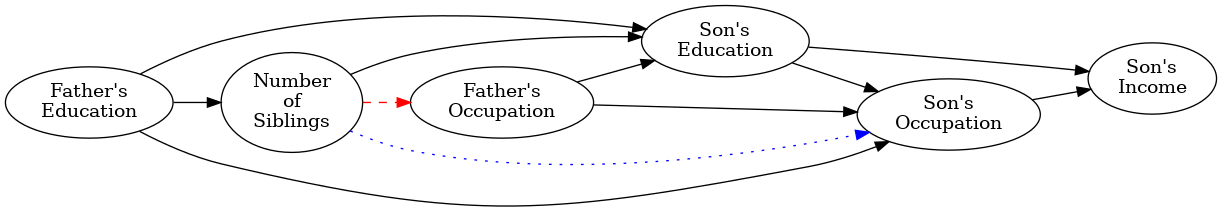}
%%%%% 1
\centering
\includegraphics[width=0.99\linewidth]{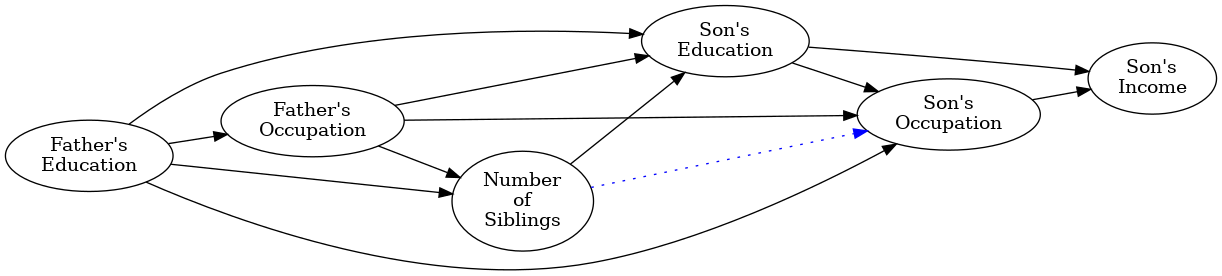}
%%%%% 2
\centering
\includegraphics[width=0.99\linewidth]{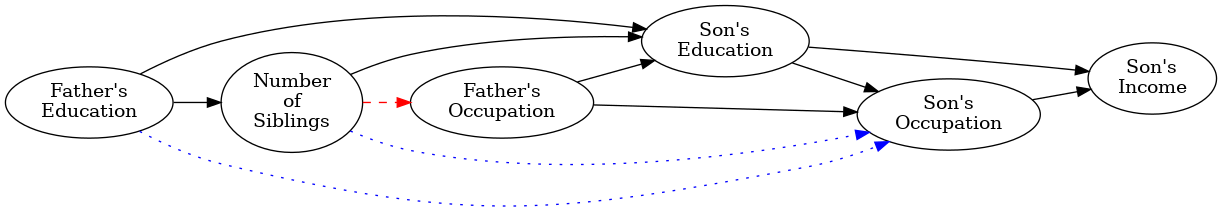}
%%%%%
\caption{
Three examples of the graphs enumerated by I-CAM-UV with specifying the following conditions: ``Son's Income'' becomes the sink, there is a causal relationship from ``Father's Education'' to ``Son's Education'', and there is a causal relationship from ``Father's Occupation'' to ``Son's Education''.
The solid black directed edges represent correctly estimated causal relationships.
The dashed red directed edges represent wrongly estimated causal relationships.
The dotted blue directed edges represent expected but not estimated causal relationships.
}
\label{fig:sam-icamuv++}
\end{figure*}

\end{document}